\documentclass[12pt]{amsart}
\usepackage{amsmath,amsthm,amssymb}
\usepackage{dsfont}
\usepackage{mathrsfs} 
\usepackage[all,cmtip]{xy}
\usepackage{patchcmd}
\input{diagxy}
\usepackage[backref=page]{hyperref}
\usepackage{bm}
\usepackage{mathtools}
\usepackage{stmaryrd}
\usepackage{enumitem}
\usepackage{tikz}
\usepackage{tikzit}

\tikzstyle{black little dot}=[fill=black, draw=black, tikzit category=string diagram, shape=circle, radius=0.1cm]
\tikzstyle{medium box}=[fill=white, draw=black, shape=rectangle, tikzit category=string diagram, minimum width=1cm, minimum height=.75cm, tikzit draw=black, tikzit shape=rectangle, thick]
\tikzstyle{small box}=[fill=white, draw=black, shape=rectangle, tikzit category=string diagram, tikzit draw=black, tikzit shape=rectangle, minimum width=.5cm, minimum height=.5cm, thick]
\tikzstyle{small circle}=[fill=white, draw=black, shape=circle, tikzit category=string diagram, tikzit draw=black, tikzit shape=circle, thick]
\tikzstyle{circle}=[fill=white, draw=black, shape=circle, tikzit category=string diagram, tikzit draw=black, tikzit shape=circle, radius=.1cm, thick]
\tikzstyle{wide small box}=[fill=white, draw=black, shape=rectangle, tikzit category=string diagram, tikzit fill=white, tikzit draw=black, minimum width=1cm, minimum height=.5cm, thick]

\tikzstyle{wire}=[-, draw=black, tikzit draw=black, thick]
\tikzstyle{dashline}=[-, tikzit category=string diagram, thick, dashed]
\tikzstyle{dotsline}=[-, dotted]
\tikzstyle{cyanthinline}=[-, draw=cyan]

\usepackage{color}
\usepackage[normalem]{ulem}
\usepackage[top=2.4cm, bottom=2.4cm, left=2cm, right=2cm]{geometry}
\usepackage{environ}
\allowdisplaybreaks

\makeatletter
\patchcommand\@starttoc{\begin{quote}}{\end{quote}}
\def\@tocline#1#2#3#4#5#6#7{\relax
  \ifnum #1>\c@tocdepth 
  \else
    \par \addpenalty\@secpenalty\addvspace{#2}%
    \begingroup \hyphenpenalty\@M
    \@ifempty{#4}{%
      \@tempdima\csname r@tocindent\number#1\endcsname\relax
    }{%
      \@tempdima#4\relax
    }%
    \parindent\z@ \leftskip#3\relax \advance\leftskip\@tempdima\relax
    \rightskip\@pnumwidth plus4em \parfillskip-\@pnumwidth
    #5\leavevmode\hskip-\@tempdima
      \ifcase #1
       \or\or \hskip 1em \or \hskip 2em \else \hskip 3em \fi%
      #6\nobreak\relax
    \dotfill\hbox to\@pnumwidth{\@tocpagenum{#7}}\par
    \nobreak
    \endgroup
  \fi}
\makeatother

 \theoremstyle{plain}
 \newtheorem{thm}{Theorem}[section]
 \newtheorem{cor}[thm]{Corollary}
 \newtheorem{lem}[thm]{Lemma}
 
 \newtheorem{prop}[thm]{Proposition}

\theoremstyle{definition}
 \newtheorem{defn}[thm]{Definition}

\theoremstyle{remark}
 \newtheorem{rem}[thm]{Remark}

 \newtheorem{ter}[thm]{Terminology}
 \newtheorem{nota}[thm]{Notation}
 
 \newtheorem{exam}[thm]{Example}

 \numberwithin{equation}{section}

\theoremstyle{plain}

 \DeclareMathOperator{\dom}{dom}
 \DeclareMathOperator{\cod}{cod}

\def\XXint#1#2#3{{\setbox0=\hbox{$#1{#2#3}{\int}$}
\vcenter{\hbox{$#2#3$}}\kern-.5\wd0}}

\newcommand{\0}{\emptyset}

\DeclareMathAlphabet{\mathpzc}{OT1}{pzc}{m}{it}

\providecommand\given{}
\newcommand\SetSymbol[1][]{%
\nonscript \: #1 \vert
\allowbreak
\nonscript\:
\mathopen{}}
\DeclarePairedDelimiterX\set[1]\{\}{%
\renewcommand\given{\SetSymbol[\delimsize]}
#1
}

\newbox\gnBoxA
\newdimen\gnCornerHgt
\setbox\gnBoxA=\hbox{$\ulcorner$}
\global\gnCornerHgt=\ht\gnBoxA
\newdimen\gnArgHgt

\def\code #1{%
        \setbox\gnBoxA=\hbox{$#1$}%
        \gnArgHgt=\ht\gnBoxA%
        \ifnum \gnArgHgt<\gnCornerHgt
                \gnArgHgt=0pt%
        \else
                \advance \gnArgHgt by -\gnCornerHgt%
        \fi
        \raise\gnArgHgt\hbox{$\ulcorner$} \box\gnBoxA %
                \raise\gnArgHgt\hbox{$\urcorner$}}

\newcommand{\fun}{\longrightarrow}

\newcommand{\sub}{\subseteq}

\newcommand{\mi}{\smallsetminus}

\newcommand{\cat}[1]{\textup{\textsf{#1}}}

\DeclareMathOperator{\cau}{\sf{Cau}}

\DeclareMathOperator{\pa}{pa}
\DeclareMathOperator{\nd}{nd}

\newcommand{\memph}[1]{{\color{magenta}\emph{#1}}}

\newcommand{\moi}{\mathbf{I}}

\newcommand{\stri}{\Gamma}

\title{Markov categories, causal theories, and the do-calculus}

\author[Y. Yin]{Yimu Yin}
\address[Y. Yin]{Pasadena, California}
\email{yimu.yin@hotmail.com}

\author[J. Zhang]{Jiji Zhang}
\address[J. Zhang]{Department of Religion and Philosophy,  Hong Kong Baptist University}
\email{zhangjiji@hkbu.edu.hk}

\begin{document}

\begin{abstract}
We give a category-theoretic treatment of causal models that formalizes the syntax for causal reasoning over a directed acyclic graph (DAG) by associating a free Markov category with  the DAG in a canonical way.  This framework enables us to define and study important concepts in causal reasoning from an abstract and ``purely causal'' point of view,  such as causal independence/separation,  causal conditionals,  and decomposition of intervention effects.  Our results regarding these concepts abstract away from the details of the commonly adopted causal models such as (recursive) structural equation models or causal Bayesian networks.  They are therefore more widely applicable and in a way conceptually clearer.  Our results are also intimately related to Judea Pearl's celebrated do-calculus,  and yield a syntactic version of a core part of the calculus that is inherited in all causal models.  In particular,  it induces a simpler and specialized version of Pearl's do-calculus in the context of causal Bayesian networks,  which we show is as strong as the full version.

\end{abstract}

\maketitle

\section{Introduction}
\label{sec:intro}

Causal models based on directed acyclic graphs (DAGs), such as recursive structural equation models \cite{GallesPearl1998, Halpern2000} and causal Bayesian networks \cite{SGS2000, pearl2009causality},  have been vigorously studied and widely applied as powerful tools for causal reasoning.  However,  from a logical point of view,  the syntax underlying such causal models is usually left implicit or even obscure in the literature.  This lacuna is fixed in recent category-theoretic work on the subject \cite{Fong:thesis, jacobs2019causal}, where the distinction between syntax and semantics is made clear in the style of F.W. Lawvere's functorial semantics \cite{lawvere1963functorial}. Specifically, the syntax is provided by a monoidal category of a certain kind induced by a DAG, and a distinguished class of morphisms therein can be viewed as {\it syntactic} causal effects, which may then be interpreted in various ways. Causal Bayesian networks, for example, interpret a causal effect of this kind with a stochastic matrix that represents probability distributions over the outcome-variables resulting from interventions on the treatment-variables. In this light, a causal Bayesian network is a functor---a structure-preserving mapping---from the syntax category to a category whose morphisms are stochastic matrices. As another example (and a degenerate case of causal Bayesian networks), deterministic structural equation models interpret a causal effect of this kind as a function that represents how the values of the outcome-variables depend on those of the treatment-variables. Thus, a (recursive) deterministic structural equation model may be viewed as a functor from the syntax category to a category whose morphisms are functions.

In this paper we build on this category-theoretic framework and study some important concepts in DAG-based causal reasoning from a syntactic and more abstract perspective. In particular, we work with the categories defined in \cite{Fong:thesis}, called {\it causal theories}, with an extra constraint to make them Markov categories in the sense of \cite{fritz2020synthetic}. We study the morphisms in that category that correspond to what we call syntactic causal effects, using the graphical language of string diagrams as vehicles for our arguments. One of our main results concerns the decomposition or disintegration of causal effect morphisms, or in the terminology of \cite{fritz2020synthetic}, the existence of a conditional for a causal effect morphism. Roughly, this refers to the property that the causal effect of $x$ on $y$ and $z$ can be decomposed into the causal effect of $x$ on $y$ followed by that of $x$ and $y$ on $z$. We derive a condition that is sufficient and necessary for decomposability in a causal theory.  Interestingly, the condition is precisely the condition in a specialized version of the second rule of Judea Pearl's do-calculus \cite{pearl2009causality}. This agreement is of course not a coincidence and has, we submit, several instructive implications. The other rules of the do-calculus have more straightforward counterparts in terms of causal effect morphisms, and the upshot is a generic do-calculus at the syntactic level.

This generic calculus, we argue, captures the ``causal core'' of reasoning about interventions, and is automatically inherited in all causal models, including but not limited to the popular probabilistic ones. In particular, it induces a simpler and specialized version of Pearl's do-calculus in the context of causal Bayesian networks. Importantly, we show that given the probability calculus, the simpler and specialized version entails the full version of the do-calculus, corroborating our contention that the generic do-calculus corresponds to the purely causal component of the well-known probabilistic do-calculus.

The rest of the paper is organized as follows. In Section \ref{sec:cat}, we introduce the basics of category theory and the intuitive language of string diagrams, leading to the notion of a Markov category. In Section \ref{sec:cau}, we define causal theories as an abstraction of causal Bayesian networks and as free Markov categories, and highlight a class of morphisms constructed in \cite{Fong:thesis}, which we call ``causal effect morphisms''. Section \ref{sec:effect} presents some results about causal effect morphisms, which yield a more abstract and syntactic counterpart to Pearl's do-calculus. We show in Section \ref{sec:cal} that the syntactic do-calculus entails a simpler and specialized version of probabilistic do-calculus, and that, despite its simplicity, the specialized version is actually as strong as the full version.

\section{Markov categories}
\label{sec:cat}
For the sake of space and readability, we will only describe the notions of category theory that are essential for understanding this paper, and introduce the axioms for a Markov category using the language of string diagrams. For readers interested in learning more about category theory and string diagrams, we recommend the canonical treatises \cite{MacLane1998} and  \cite{selinger2010survey}, among other excellent textbooks and surveys.

A category $\cat C$ consists of two types of entities: \memph{objects} $A$, $B$, $C$, $\ldots$ and \memph{arrows} $f$, $g$, $h$, $\ldots$, subject to the following three rules:
\begin{itemize}
 \item For each arrow $f$ there are given two objects $\dom(f)$ and $\cod(f)$, called the \memph{domain} and the \memph{codomain} of $f$. We usually write $f :A \fun B$ to indicate that $A = \dom(f)$ and  $B = \cod(f)$.
 \item Given two arrows $f : A \fun B$ and $g: B \fun C$, that is, $\cod(f) = \dom(g)$, there is a third arrow $g \circ f : A \fun C$, called the \memph{composition} of $f$ and $g$.
 \item  For each object $A$ there is an arrow $1_A : A \fun A$, called the \memph{identity} or \memph{unit} arrow of $A$.
\end{itemize}
In addition, the obvious \memph{unitality} and \memph{associativity} laws hold for compositions: for all $f : A \fun B$, $g : B \fun C$, and $h : C \fun D$,
\[
1_B \circ f = f, \quad f \circ 1_A = f, \quad (h \circ g) \circ f = h \circ (g \circ f).
\]

An arrow in category theory is also called a \memph{morphism} or a \memph{map}. Here is a more formal definition:

\begin{defn}\label{cat:def:2}
Let $\cat C$ be a quadruple $(\cat C_0, \cat C_1, \dom, \cod)$, where $\cat C_0$ is referred to as a class of \memph{objects},  $\cat C_1$ is referred to as a class of \memph{morphisms}, and $\dom : \cat C_1 \fun \cat C_0$, $\cod : \cat C_1 \fun \cat C_0$ are functions. A morphism $f$ in $\cat C_1$ is usually written as $f : A \fun B$ with $\dom(f) = A$ and $\cod(f) = B$. For each pair of objects $A$, $B$ in $\cat C_0$, the class of all morphisms $f$ with  $\dom(f) = A$ and $\cod(f) = B$ is denoted by $\hom_{\cat C}(A,B)$.

Let $\cat C_2 = \set{(f, g) \in \cat C_1 \times \cat C_1 \given \cod(f) = \dom (g)}$. We say that $\cat C$ is a \memph{category} if it also comes with a morphism $1_A : A \fun A$ for every $A \in \cat C_0$, called the \memph{identity morphism} of $A$, and a function $\circ : \cat C_2 \fun \cat C_1$, called \memph{composition}, subject to the associativity and unitality laws given above.
\end{defn}

Often we just write $x \in \cat C$ and let the context determine whether $x$ is an object or a morphism.

A paradigmatic example of a category is the category \cat{Set}, containing sets as objects and functions as morphisms. In this category, the composition of morphisms is just the composition of functions and for each object $A$, the identity morphism is just the identity function.

It is helpful to think of a morphism as an abstract function, or a box with input wires and output wires, as in the graphical language of string diagrams.
The four rudiments of a category are depicted in such a graphical language as follows:
\begin{equation}\label{free:graph:def}
\begin{tikzpicture}[xscale = .5, yscale = .5, baseline=(current  bounding  box.center)]
\begin{pgfonlayer}{nodelayer}
		\node [style=none] (42) at (2.6, 2.5) {$A$};
		\node [style=none] (110) at (10.275, 2.15) {};
		\node [style=none] (113) at (10.275, 0.825) {};
		\node [style=small box] (156) at (10.275, 2.5) {$f$};
		\node [style=none] (157) at (10.275, 4.3) {};
		\node [style=none] (158) at (10.275, 2.925) {};
		\node [style=small box] (166) at (24.475, 1.75) {$f$};
		\node [style=none] (167) at (24.475, 3.65) {};
		\node [style=none] (168) at (24.475, 2.175) {};
		\node [style=none] (192) at (3.125, 3.7) {};
		\node [style=none] (193) at (3.125, 1.3) {};
		\node [style=none] (194) at (3.125, -0.5) {object $A$};
		\node [style=none] (195) at (9.725, 1.2) {$A$};
		\node [style=none] (196) at (9.725, 3.8) {$B$};
		\node [style=none, label={[align=center]center:morphism \\$f : A \fun B$}] (197) at (10.275, -0.75) {};
		\node [style=none] (198) at (17.3, 2.5) {$A$};
		\node [style=none] (199) at (17.875, 3.7) {};
		\node [style=none] (200) at (17.875, 1.3) {};
		\node [style=none] (201) at (17.875, -0.5) {identity $1_A$};
		\node [style=small box] (203) at (24.475, 4) {$g$};
		\node [style=none] (204) at (24.475, 5.65) {};
		\node [style=none] (205) at (24.475, 4.425) {};
		\node [style=none] (206) at (24.475, 1.4) {};
		\node [style=none] (207) at (24.475, 0.175) {};
		\node [style=none, label={[align=center]center:composition \\ $g \circ f$}] (208) at (24.475, -1.3) {};
		\node [style=none] (209) at (23.975, 0.5) {$A$};
		\node [style=none] (210) at (23.975, 2.925) {$B$};
		\node [style=none] (211) at (23.975, 5.05) {$C$};
	\end{pgfonlayer}
	\begin{pgfonlayer}{edgelayer}
		\draw [style=wire] (110.center) to (113.center);
		\draw [style=wire] (157.center) to (158.center);
		\draw [style=wire] (167.center) to (168.center);
		\draw [style=wire] (192.center) to (193.center);
		\draw [style=wire] (199.center) to (200.center);
		\draw [style=wire] (204.center) to (205.center);
		\draw [style=wire] (206.center) to (207.center);
	\end{pgfonlayer}
\end{tikzpicture}
\end{equation}
Note that string diagrams are parsed in the lower-left to upper-right order.

\begin{rem}\label{intro:string}
A string diagram is a topological graph  in which every edge is labelled with an object and every vertex with a morphism \cite{selinger2010survey}. Object labels such as $A$, $B$ are usually omitted except when they are needed for clarity or emphasis. A labelled vertex is also called a \memph{node}, and is often drawn as a box such as
$\begin{tikzpicture}[xscale = 1, yscale = 1, baseline={([yshift=7pt]current bounding box.south)}]
   \begin{pgfonlayer}{nodelayer}
		\node [style=small box] (0) at (0, 0) {$f$};
	\end{pgfonlayer}
\end{tikzpicture}$ for readability. Just as in the usual symbolic formalism, a morphism $f$ may be represented by many string diagrams.
\end{rem}

Categories may serve as objects in a ``higher'' category, and the morphisms between categories are known as functors:

\begin{defn}\label{def:func}
Let $\cat C$, $\cat D$ be categories and $F$ a pair of mappings $F_0 : \cat C_0 \fun \cat D_0$, $F_1 : \cat C_1 \fun \cat D_1$. Then $F$ is a \memph{functor}, written as $F : \cat C \fun \cat D$,  if the following three conditions, corresponding to the three conditions for a category, are satisfied:
\begin{itemize}
 \item $F$ preserves domains and codomains, that is, $F_1(f: A \fun B)$ is a morphism $F_0(A) \fun F_0(B)$ for all morphisms $f \in \cat C_1$.
 \item  $F$ preserves compositions, that is, $F_1(g \circ f) = F_1(g) \circ F_1(f)$ for all compositions $g \circ f \in \cat C_1$.
 \item  $F$ preserves identities, that is, $F_1(1_A) = 1_{F_0(A)}$ for all objects $A \in \cat C_0$.
\end{itemize}
\end{defn}
Compositions of functors may be defined using composition of mappings. Then it is routine to check that categories and functors  form a ``higher'' category.

We now introduce more structures needed for our purpose. Start with the (strict) monoidal structure:
\begin{defn}\label{defn:moncat}
A \memph{strict monoidal category} is a category $\cat C$ equipped with a functor $\otimes : \cat C \times \cat C \to \cat C$, called the \memph{monoidal product}, and a distinguished object $\moi$, called the \memph{monoidal unit} (of the monoidal product), that satisfy
associativity and unitality:
\begin{equation}\label{monoidal:def}
     (A\otimes B)\otimes C = A\otimes (B\otimes C), \quad
    \moi\otimes A = A = A\otimes \moi
  \end{equation}
\end{defn}

Many commonly encountered monoidal categories are actually not strict because equation \ref{monoidal:def} holds only ``up to isomorphism.'' For example, the category \cat{Set} has an obvious monoidal product, which is just the Cartesian product (of sets and of functions). The monoidal unit is any singleton set, but unitality is a matter of isomorphism rather than strict identity. The formal definition of a (possibly non-strict) monoidal category is rather more complex and requires the notion of a natural transformation, which we omit to keep things simple. The monoidal categories we will focus on in this paper are strict.\footnote{By S. Mac Lane's coherence theorem \cite[Theorem XI.3.1]{MacLane1998}, every monoidal category is monoidally equivalent to a strict
monoidal category. So there is a sense in which we can treat monoidal categories as if they are all strict (even though they are not). See \cite{selinger2010survey} for more on how this is justified.}

Since the monoidal product is a functor, it applies to both objects and morphisms in the category. Thus the graphical syntax in (\ref{free:graph:def}) is extended for monoidal categories as follows:
\begin{equation}\label{mon:free:graph:def}
\begin{tikzpicture}[xscale = .55, yscale = .5, baseline=(current  bounding  box.center)]
\begin{pgfonlayer}{nodelayer}
		\node [style=none] (42) at (1.6, 2.5) {$A$};
        \node [style=none] (212) at (3.1, 2.5) {$B$};
		\node [style=none] (110) at (16.9, 2.15) {};
		\node [style=none] (113) at (16.9, 0.825) {};
		\node [style=wide small box] (156) at (17.525, 2.5) {$f$};
		\node [style=none] (157) at (16.9, 4.3) {};
		\node [style=none] (158) at (16.9, 2.925) {};
		\node [style=small box] (166) at (24.475, 2.475) {$f$};
		\node [style=none] (167) at (24.475, 4.125) {};
		\node [style=none] (168) at (24.475, 2.9) {};
		\node [style=none] (192) at (2.125, 3.7) {};
		\node [style=none] (193) at (2.125, 1.3) {};
		\node [style=none, label={[align=center]center:monoidal product\\of objects}] (194) at (2.875, -0.5) {};
		\node [style=none] (195) at (16.35, 1.2) {$A_1$};
		\node [style=none] (196) at (16.35, 3.8) {$B_1$};
		\node [style=none, label={[align=center]center:monoidal unit\\(empty diagram)}] (197) at (10.275, -0.475) {};
		\node [style=none, label={[align=center]center:morphism\\$f : \bigotimes_i A_i \fun \bigotimes_j B_j$}] (201) at (17.625, -1) {};
		\node [style=none] (206) at (24.475, 2.125) {};
		\node [style=none] (207) at (24.475, 0.9) {};
		\node [style=none, label={[align=center]center:monoidal product\\of morphisms}] (208) at (25.275, -1) {};
		\node [style=none] (209) at (23.975, 1.225) {$A$};
		\node [style=none] (210) at (23.975, 3.55) {$B$};
		\node [style=none] (213) at (3.625, 3.7) {};
		\node [style=none] (214) at (3.625, 1.3) {};
		\node [style=none] (215) at (18.15, 2.15) {};
		\node [style=none] (216) at (18.15, 0.825) {};
		\node [style=none] (217) at (18.15, 4.3) {};
		\node [style=none] (218) at (18.15, 2.925) {};
		\node [style=none] (219) at (18.75, 1.2) {$A_n$};
		\node [style=none] (220) at (18.75, 3.8) {$B_m$};
		\node [style=none] (221) at (17.6, 1.2) {$\cdots$};
		\node [style=none] (222) at (17.6, 3.8) {$\cdots$};
		\node [style=small box] (224) at (25.975, 2.5) {$g$};
		\node [style=none] (225) at (25.975, 4.15) {};
		\node [style=none] (226) at (25.975, 2.925) {};
		\node [style=none] (227) at (25.475, 3.55) {$D$};
		\node [style=none] (228) at (25.975, 2.125) {};
		\node [style=none] (229) at (25.975, 0.9) {};
		\node [style=none] (230) at (25.475, 1.225) {$C$};
	\end{pgfonlayer}
	\begin{pgfonlayer}{edgelayer}
		\draw [style=wire] (110.center) to (113.center);
		\draw [style=wire] (157.center) to (158.center);
		\draw [style=wire] (167.center) to (168.center);
		\draw [style=wire] (192.center) to (193.center);
		\draw [style=wire] (206.center) to (207.center);
		\draw [style=wire] (213.center) to (214.center);
		\draw [style=wire] (215.center) to (216.center);
		\draw [style=wire] (217.center) to (218.center);
		\draw [style=wire] (225.center) to (226.center);
		\draw [style=wire] (228.center) to (229.center);
	\end{pgfonlayer}
\end{tikzpicture}
\end{equation}

\begin{nota}
We denote by $A^n$ the monoidal product of $n$ copies of an object  $A$ itself; this includes the empty product $A^0 = \moi$. When an object of the form $A_1 \otimes \ldots \otimes A_k$ is introduced in the discussion, the indices are in general meant to indicate the ordering in which the monoidal product is taken.
\end{nota}
For the present purpose, the monoidal structure is especially useful because it can be used to express causal processes or mechanisms that run in parallel, as is visualized in (\ref{mon:free:graph:def}), whereas composition is used to express those that run in sequence.

A \memph{symmetric} monoidal category is a monoidal category with natural isomorphisms
$\sigma_{AB} : A \otimes B \cong B \otimes A$ that satisfy certain coherence conditions (the details do not matter for the present purpose). Graphically, a symmetry isomorphism is depicted as a crossing:
\begin{equation}\label{sym:graph}
\begin{tikzpicture}[xscale = .5, yscale = .4, baseline=(current  bounding  box.center)]
\begin{pgfonlayer}{nodelayer}
		\node [style=none] (110) at (-2.475, 4.6) {};
		\node [style=none] (112) at (-2.475, 3.775) {};
		\node [style=none] (139) at (-2.475, 6.625) {};
		\node [style=none] (141) at (-2.475, 5.775) {};
		\node [style=none] (246) at (-0.775, 4.55) {};
		\node [style=none] (247) at (-0.775, 3.775) {};
		\node [style=none] (249) at (-0.775, 6.625) {};
		\node [style=none] (250) at (-0.775, 5.775) {};
		\node [style=none] (254) at (-0.25, 4.125) {$B$};
		\node [style=none] (255) at (-3, 4.125) {$A$};
		\node [style=none] (256) at (-0.25, 6.25) {$A$};
		\node [style=none] (257) at (-3, 6.25) {$B$};
		\node [style=none, label={[align=center]center:symmetry $\sigma_{AB}$}] (258) at (-1.5, 2.8) {};
	\end{pgfonlayer}
	\begin{pgfonlayer}{edgelayer}
		\draw [style=wire] (112.center) to (110.center);
		\draw [style=wire] (141.center) to (139.center);
		\draw [style=wire] (247.center) to (246.center);
		\draw [style=wire] (250.center) to (249.center);
		\draw [style=wire, in=90, out=-90, looseness=0.75] (250.center) to (110.center);
		\draw [style=wire, in=90, out=-90, looseness=0.75] (141.center) to (246.center);
	\end{pgfonlayer}
\end{tikzpicture}
\end{equation}

Finally, we can define a Markov category, following the lead of \cite{fritz2020synthetic}.
\begin{defn}
A \memph{Markov category} is a symmetric monoidal category such that for each  object $A$ in $\cat C$ there are distinguished morphisms $\delta_A : A \fun A \otimes A$, called the \memph{duplicate} on $A$, and $\epsilon_A : A \fun \moi$, called the \memph{discard} on $A$, that satisfy the following:
\begin{equation}\label{mon:law}
\begin{tikzpicture}[xscale = .65, yscale = .6, baseline=(current  bounding  box.center)]
\begin{pgfonlayer}{nodelayer}
		\node [style=none] (139) at (-1.625, 5.45) {};
		\node [style=none] (141) at (-1.625, 5.325) {};
		\node [style=none] (144) at (-2.1, 6.125) {};
		\node [style=none] (146) at (-2.1, 5.825) {};
		\node [style=none] (246) at (-0.925, 4.8) {};
		\node [style=none] (247) at (-0.925, 4.425) {};
		\node [style=none] (249) at (-0.2, 6.1) {};
		\node [style=none] (250) at (-0.2, 5.325) {};
		\node [style=none] (251) at (-1.15, 6.125) {};
		\node [style=none] (252) at (-1.15, 5.825) {};
		\node [style=none] (281) at (1.15, 5.175) {$=$};
		\node [style=none] (298) at (3.85, 5.45) {};
		\node [style=none] (299) at (3.85, 5.325) {};
		\node [style=none] (300) at (4.325, 6.125) {};
		\node [style=none] (301) at (4.325, 5.825) {};
		\node [style=none] (302) at (3.15, 4.8) {};
		\node [style=none] (303) at (3.15, 4.425) {};
		\node [style=none] (304) at (2.425, 6.1) {};
		\node [style=none] (305) at (2.425, 5.325) {};
		\node [style=none] (306) at (3.375, 6.125) {};
		\node [style=none] (307) at (3.375, 5.825) {};
		\node [style=none] (314) at (9.6, 5.175) {$=$};
		\node [style=none] (315) at (11.875, 5.175) {$=$};
		\node [style=none] (316) at (10.75, 6.1) {};
		\node [style=none] (317) at (10.75, 4.425) {};
		\node [style=none] (318) at (13.05, 6) {$\bullet$};
		\node [style=none] (319) at (13.05, 5.575) {};
		\node [style=none] (320) at (13.75, 5.05) {};
		\node [style=none] (321) at (13.75, 4.425) {};
		\node [style=none] (322) at (14.475, 6.1) {};
		\node [style=none] (323) at (14.475, 5.575) {};
		\node [style=none] (308) at (15.95, 6.25) {};
		\node [style=none] (309) at (17.375, 5.2) {};
		\node [style=none] (310) at (16.675, 4.65) {};
		\node [style=none] (311) at (16.675, 4.275) {};
		\node [style=none] (312) at (17.375, 6.25) {};
		\node [style=none] (313) at (15.95, 5.2) {};
		\node [style=none] (324) at (8.4, 6) {$\bullet$};
		\node [style=none] (325) at (8.4, 5.575) {};
		\node [style=none] (326) at (7.7, 5.05) {};
		\node [style=none] (327) at (7.7, 4.425) {};
		\node [style=none] (328) at (6.975, 6.1) {};
		\node [style=none] (329) at (6.975, 5.575) {};
		\node [style=none] (330) at (20.125, 6.1) {};
		\node [style=none] (331) at (20.125, 5.575) {};
		\node [style=none] (332) at (20.825, 5.05) {};
		\node [style=none] (333) at (20.825, 4.425) {};
		\node [style=none] (334) at (21.55, 6.1) {};
		\node [style=none] (335) at (21.55, 5.575) {};
		\node [style=none] (336) at (18.8, 5.175) {$=$};
		\node [style=none, label={[align=center]center: coassociativity\\$(\delta_A \otimes 1_A) \circ \delta_A =  (1_A \otimes \delta_A) \circ \delta_A$}] (337) at (1.25, 3.175) {};
		\node [style=none, label={[align=center]center:counitality\\$(1_A \otimes \epsilon_A) \circ \delta_A = 1_A = (\epsilon_A \otimes 1_A) \circ \delta_n$}] (338) at (10.95, 3.175) {};
		\node [style=none, label={[align=center]center:cocommutativity\\$\sigma_{AA} \circ \delta_A = \delta_A$}] (339) at (18.8, 3.175) {};
	\end{pgfonlayer}
	\begin{pgfonlayer}{edgelayer}
		\draw [style=wire] (141.center) to (139.center);
		\draw [style=wire] (146.center) to (144.center);
		\draw [style=wire] (247.center) to (246.center);
		\draw [style=wire] (250.center) to (249.center);
		\draw [style=wire] (252.center) to (251.center);
		\draw [style=wire, bend right=90, looseness=1.25] (141.center) to (250.center);
		\draw [style=wire, bend right=90, looseness=1.25] (146.center) to (252.center);
		\draw [style=wire] (299.center) to (298.center);
		\draw [style=wire] (301.center) to (300.center);
		\draw [style=wire] (303.center) to (302.center);
		\draw [style=wire] (305.center) to (304.center);
		\draw [style=wire] (307.center) to (306.center);
		\draw [style=wire, bend left=90, looseness=1.25] (299.center) to (305.center);
		\draw [style=wire, bend left=90, looseness=1.25] (301.center) to (307.center);
		\draw [style=wire] (311.center) to (310.center);
		\draw [style=wire] (317.center) to (316.center);
		\draw [style=wire] (319.center) to (318.center);
		\draw [style=wire] (321.center) to (320.center);
		\draw [style=wire] (323.center) to (322.center);
		\draw [style=wire, bend right=90, looseness=1.25] (319.center) to (323.center);
		\draw [style=wire, in=-90, out=90] (309.center) to (308.center);
		\draw [style=wire] (311.center) to (310.center);
		\draw [style=wire, in=-90, out=90, looseness=0.75] (313.center) to (312.center);
		\draw [style=wire, bend left=90, looseness=1.25] (309.center) to (313.center);
		\draw [style=wire] (325.center) to (324.center);
		\draw [style=wire] (327.center) to (326.center);
		\draw [style=wire] (329.center) to (328.center);
		\draw [style=wire, bend left=90, looseness=1.25] (325.center) to (329.center);
		\draw [style=wire] (331.center) to (330.center);
		\draw [style=wire] (333.center) to (332.center);
		\draw [style=wire] (335.center) to (334.center);
		\draw [style=wire, bend right=90, looseness=1.25] (331.center) to (335.center);
	\end{pgfonlayer}
\end{tikzpicture}
\end{equation}

\begin{equation}\label{fins:frob}
\begin{tikzpicture}[xscale = .55, yscale = .45, baseline=(current  bounding  box.center)]
\begin{pgfonlayer}{nodelayer}
		\node [style=none] (42) at (7.725, -0.75) {$=$};
		\node [style=none] (108) at (4.275, 0.4) {};
		\node [style=none] (110) at (6.275, 0.4) {};
		\node [style=none] (111) at (5.275, -1.825) {};
		\node [style=none] (112) at (4.275, -0.35) {};
		\node [style=none] (113) at (6.275, -0.325) {};
		\node [style=none] (114) at (5.275, -1.075) {};
		\node [style=none] (117) at (9.225, 0.65) {};
		\node [style=none] (118) at (12.125, 0.65) {};
		\node [style=none] (119) at (10.05, -1.825) {};
		\node [style=none] (120) at (9.225, -0.625) {};
		\node [style=none] (121) at (10.85, -0.6) {};
		\node [style=none] (122) at (10.05, -1.225) {};
		\node [style=none] (124) at (10.825, 0.65) {};
		\node [style=none] (125) at (13.825, 0.65) {};
		\node [style=none] (126) at (12.95, -1.825) {};
		\node [style=none] (127) at (12.125, -0.625) {};
		\node [style=none] (128) at (13.825, -0.6) {};
		\node [style=none] (129) at (12.95, -1.25) {};
		\node [style=none] (131) at (19.65, -0.75) {$=$};
		\node [style=none] (134) at (18.2, -1.525) {};
		\node [style=none] (137) at (18.2, -0.05) {$\bullet$};
		\node [style=none] (138) at (20.25, -3) {$\epsilon_{A \otimes B} = \epsilon_{A} \otimes \epsilon_{B}$};
		\node [style=none] (145) at (21.025, -1.525) {};
		\node [style=none] (146) at (21.025, -0.05) {$\bullet$};
		\node [style=none] (148) at (22.275, -1.525) {};
		\node [style=none] (149) at (22.275, -0.05) {$\bullet$};
		\node [style=none] (151) at (9.25, -3) {$\delta_{A \otimes B} = (1_A \otimes \sigma_{BA} \otimes 1_B) \circ (\delta_{A} \otimes \delta_B)$};
	\end{pgfonlayer}
	\begin{pgfonlayer}{edgelayer}
		\draw [style=wire, bend right=90, looseness=1.25] (112.center) to (113.center);
		\draw [style=wire] (108.center) to (112.center);
		\draw [style=wire] (110.center) to (113.center);
		\draw [style=wire] (111.center) to (114.center);
		\draw [style=wire, bend right=90, looseness=1.25] (120.center) to (121.center);
		\draw [style=wire] (117.center) to (120.center);
		\draw [style=wire, in=90, out=-90, looseness=1.25] (118.center) to (121.center);
		\draw [style=wire] (119.center) to (122.center);
		\draw [style=wire, bend right=90, looseness=1.25] (127.center) to (128.center);
		\draw [style=wire, in=90, out=-90, looseness=1.25] (124.center) to (127.center);
		\draw [style=wire] (125.center) to (128.center);
		\draw [style=wire] (126.center) to (129.center);
		\draw [style=wire] (134.center) to (137.center);
		\draw [style=wire] (145.center) to (146.center);
		\draw [style=wire] (148.center) to (149.center);
	\end{pgfonlayer}
\end{tikzpicture}
\end{equation}
\begin{equation}\label{discar:nat}
\begin{tikzpicture}[xscale = .55, yscale = .45, baseline=(current  bounding  box.center)]
	\begin{pgfonlayer}{nodelayer}
		\node [style=none] (200) at (19.4, 1.75) {$=$};
		\node [style=none] (202) at (17.45, 1.525) {};
		\node [style=none] (205) at (17.45, 0.6) {};
		\node [style=small box] (207) at (17.45, 1.875) {$f$};
		\node [style=none] (208) at (17.45, 3.225) {$\bullet$};
		\node [style=none] (209) at (17.45, 2.3) {};
		\node [style=none] (210) at (21.15, 2.725) {$\bullet$};
		\node [style=none] (211) at (21.15, 1.1) {};
		\node [style=none] (212) at (19.4, -0.75) {$\epsilon_B \circ f = \epsilon_A$};
	\end{pgfonlayer}
	\begin{pgfonlayer}{edgelayer}
		\draw [style=wire] (202.center) to (205.center);
		\draw [style=wire] (208.center) to (209.center);
		\draw [style=wire] (210.center) to (211.center);
	\end{pgfonlayer}
\end{tikzpicture}
\end{equation}
\end{defn}

\begin{nota}
Recall that our convention is to draw a string diagram in the lower-left to upper-right direction. So, above, the duplicate $\delta_A : A \fun A \otimes A$ is depicted as an upward fork
$\begin{tikzpicture}[xscale = .3, yscale = .3, baseline={([yshift=5pt]current bounding box.south)}]
   \begin{pgfonlayer}{nodelayer}
		\node [style=none] (324) at (8.4, 6.1) {};
		\node [style=none] (325) at (8.4, 5.575) {};
		\node [style=none] (326) at (7.7, 5.05) {};
		\node [style=none] (327) at (7.7, 4.425) {};
		\node [style=none] (328) at (6.975, 6.1) {};
		\node [style=none] (329) at (6.975, 5.575) {};
	\end{pgfonlayer}
	\begin{pgfonlayer}{edgelayer}
		\draw [style=wire] (325.center) to (324.center);
		\draw [style=wire] (327.center) to (326.center);
		\draw [style=wire] (329.center) to (328.center);
		\draw [style=wire, bend left=90, looseness=1.25] (325.center) to (329.center);
	\end{pgfonlayer}
\end{tikzpicture}$
and the discard $\epsilon_A : A \fun \moi$ an upward dead-end
$\begin{tikzpicture}[xscale = .3, yscale = .3, baseline={([yshift=5pt]current bounding box.south)}]
   \begin{pgfonlayer}{nodelayer}
		\node [style=none] (326) at (7.7, 5.8) {$\bullet$};
		\node [style=none] (327) at (7.7, 4.425) {};
	\end{pgfonlayer}
	\begin{pgfonlayer}{edgelayer}
		\draw [style=wire] (327.center) to (326.center);
	\end{pgfonlayer}
\end{tikzpicture}$.
\end{nota}

Thus a Markov category is endowed with both a symmetric monoidal structure and additional duplicate morphisms and discard morphisms that satisfy (\ref{mon:law})-(\ref{discar:nat}). For our purpose, duplicate morphisms are needed mainly to express the same input entering several causal processes, and discard morphisms are needed to express ignoring or marginalizing over some outcomes of a causal process. The equations in (\ref{mon:law}) are axioms for the so-called \memph{cocommutative monoidal comonoid}, and the equations in (\ref{fins:frob}) express the condition that duplicates and discards respect the monoidal product. All these axioms are fairly intuitive.

The equation in (\ref{discar:nat}) says roughly that discarding the output of a morphism is identical to discarding the input in the first place. This condition is not explicitly imposed in \cite{Fong:thesis} but is actually needed for a main result therein (more on this later). Actually (\ref{discar:nat})  holds automatically if the monoidal unit $\moi$ is terminal; conversely, (\ref{discar:nat}), together with the other conditions stipulated above, implies that  $\moi$ is terminal, and hence $\epsilon_\moi  = 1_\moi$. Also note that $\delta_\moi = 1_\moi$  in any strict Markov category. See \cite[Remark~2.3]{fritz2020synthetic}.

For categories with additional structures, the notion of a functor can be strengthened accordingly. For example, a \memph{(strong) Markov functor} between two Markov categories is a functor that preserves the relevant structures (see \cite[Definition~10.14]{fritz2020synthetic} for details.)

\section{Causal theories and effect morphisms}
\label{sec:cau}

In this section we describe the central object of study in this paper, a syntax category for DAG-based causal models defined in \cite{Fong:thesis}, called a causal theory. The framework is an abstraction of causal Bayesian networks, so we first review the latter in \ref{subsec:cbn}, and then introduce causal theories as free Markov categories in \ref{subsec:cau}. In \ref{subsec:effect} we highlight a class of morphisms constructed in \cite{Fong:thesis}, which we refer to as causal effect morphisms. Our main results are concerned with these morphisms.

\subsection{Causal Bayesian networks} \label{subsec:cbn}
 A \memph{directed graph} is a quadruple  $G = (V, A, s, t)$, where $V$, $A$ are sets and $s : A \fun V$, $t: A \fun V$ are functions. Elements in $V$ are called \memph{vertices} of $G$ and those in $A$ are \memph{directed edges} or \memph{arrows} of $G$. For $a\in A$, $s(a)$ is the \memph{source} of $a$ and $t(a)$ the \memph{target} of $a$. $G$ is \memph{finite} if both $V$ and $A$ are. It is \memph{simple} if, for all $a, a'\in A$, we have $a=a'$ whenever $s(a)=s(a')$ and $t(a)=t(a')$. We consider only finite simple graphs in this paper. A sequence of distinct arrows $a_1, \ldots, a_n \in A(G)$ is a \memph{directed path}, starting from $s(a_1)$ and ending at $t(a_n)$, if $s(a_{i+1}) = t(a_i)$ for $1 \leq i < n$. It is a \memph{cycle} if, in addition, $s(a_1) = t(a_n)$. $G$ is \memph{acyclic} if it contains no cycle. For $x, y\in V(G)$, $x$ is called a \memph{parent} of $y$ and $y$ a \memph{child} of $x$ if for some $a\in A(G)$, $s(a)=x$ and $t(a)=y$.

 A \memph{Bayesian network} (BN) over a set of (categorical) random variables $\mathbf{V}$ consists of a triple $(G, P, \upsilon)$, where $G$ is a directed acyclic graph (DAG), $P$ is a joint probability law of $\mathbf{V}$, and $\upsilon: V(G)\rightarrow \mathbf{V}$ is a bijection between the vertices of $G$ and the random variables. Following common practice, we will leave the bijection $\upsilon$ implicit and simply identify $V(G)$ with $\mathbf{V}$, and call $G$ a DAG over $\mathbf{V}$. The defining condition of a BN is that $G$ and $P$ satisfy a Markov condition, which requires that $P$ can be factorized according to $G$ as follows: \begin{equation}
 \label{fac:pre}
 P(\mathbf{V}) = \prod_{X\in \mathbf{V}}P(X | \pa_{G}(X)),
\end{equation} where $\pa_G(X)$ denotes the set of parents of $X$ in $G$. When $G$ is sufficiently sparse, the factorization enables efficient computations of various probabilities entailed by the joint probability law, which makes the BN useful for probabilistic reasoning \cite{Pearl1988}.

 The DAG in a BN usually lends itself to a causal interpretation, as a representation of the qualitative causal structure. With this causal reading, the BN framework can be extended to handle reasoning about effects of interventions \cite{SGS2000, pearl2009causality}. Specifically, a \memph{causal Bayesian network} (CBN) over $\mathbf{V}$ does not represent just one joint probability law, but a number of \memph{interventional} probability distributions. The standard setup is that for every subset $\mathbf{T}\subseteq \mathbf{V}$ and every possible value configuration $\mathbf{t}$ for $\mathbf{T}$, there is a probability distribution resulting from an (exogenous) intervention that forces $\mathbf{T}$ to take value $\mathbf{t}$. Such an interventional distribution, denoted as $P(\mathbf{V} | do(\mathbf{T}=\mathbf{t}))$ using Pearl's (\cite{pearl2009causality}) do-operator, is assumed to be equal to a truncated factorization:
\begin{equation}
\label{fac:post}
P(\mathbf{V} | do(\mathbf{T}=\mathbf{t})) = \left\{
\begin{array}{ll}
\prod_{X\in \mathbf{V}\backslash \mathbf{T}}P(X | \pa_G(X)) & \textrm{for values of } \mathbf{V} \textrm{ consistent with } \mathbf{T} = \mathbf{t}, \\
0 & \textrm{for values of } \mathbf{V} \textrm{ inconsistent with } \mathbf{T} = \mathbf{t}.
\end{array}\right.
\end{equation}
As a special case, when $\mathbf{T}=\emptyset$, we obtain the factorization (\ref{fac:pre}) of the \memph{pre-intervention} distribution. Equation (\ref{fac:post}) can be viewed as the defining axiom for the CBN, sometimes referred to as the \memph{intervention principle} \cite{zhang2011intervention}.

Note two key ideas in this formulation of a CBN: (1) for each $X\in \mathbf{V}$, $P(X|\pa_G(X))$ encodes a \memph{modular} causal process or mechanism (when $\pa_G(X)=\emptyset$, $P(X)$ is taken to encode an \memph{exogenous} mechanism for $X$), and the whole causal system is composed of these causal modules; (2) an intervention breaks the modules for its target variables but does not affect the other modules (hence the truncated factorization.) Put this way, $P(X|\pa_G(X))$ is a particular, probabilistic model of the causal module; the causal theory, as a syntax category, will express the causal module more abstractly, to which we now turn.

\subsection{Causal theories as free Markov categories}
\label{subsec:cau}
We now follow \cite{Fong:thesis} to define the causal theory induced by a DAG $G = (V, A, s, t)$, a category denoted as $\cat{Cau}(G)$. The objects are given by words over $V$. A \memph{word} over $V$ is a finite sequence of elements of $V$, and this also includes the empty word $\0$. Let $W(V)$ be the set of words over $V$. Obviously $W(V)$ is closed under \memph{concatenation}: if $v, w \in W(V)$ then $vw \in W(V)$.  So concatenation provides a monoidal product on $W(V)$, with the empty word $\0$ as the unit.

\begin{ter}\label{graph:bas}
For convenience, elements of $W(V)$ are also referred to as \memph{variables} and those of length $1$, that is, the vertices in $V$, are called \memph{atomic variables}. To ease the notation, we will henceforth denote all variables with lower case letters. Concatenation of two variables $v$, $w$ is also written as $v \otimes w$.

An atomic variable $v$ is a \memph{path ancestor} of an atomic variable $w$ if there is a directed path in $G$ from $v$ to $w$, and is an \memph{ancestor} of $w$ if it is a path ancestor of $w$ or is identical with $w$.

If no atomic variable occurs more than once in a variable $v$ then $v$ is \memph{singular}; in particular, $\0$ is singular. A singular variable is \memph{maximal} if  each atomic variable occurs exactly once in it.

Let $v = \bigotimes_{1 \leq i \leq n} v_i$, where each $v_i$ is atomic. Let $v_S = \bigotimes_{i \in S} v_i$ for $S \sub \{1, \ldots, n\}$; set $v_\0 = \0$. Write $w \sub v$, or $w \in v$ if $w$ is  atomic, and say that $w$ is a \memph{sub-variable} of $v$ if  $w$ is equal to some $v_S$. Let $w = v_S$ and $w' = v_{S'}$. Write $v/w = v_{\{1, \ldots, n\} \mi S}$, $w \cap w' = v_{S \cap S'}$, $w \mi w' = v_{S \mi S'}$, and so on. We say that $w$, $w'$ are \memph{disjoint} if no atomic variable occurs in both of them.  Note that being disjoint is not the same as $w \cap w' = \0$,  unless $v$ is singular.
\end{ter}

The morphisms in $\cat{Cau}(G)$ are generated from two distinct classes of generators (basic morphisms), in addition to the identity morphisms:
\begin{itemize}
  \item The first class consists of duplicate and discard morphisms for each atomic variable $v$
  \[
  \0 \toleft^{\epsilon_v} v \to^{\delta_v} vv.
   \]
   As mentioned previously, duplicate morphisms are needed to express the same variable entering multiple causal processes, and discard morphisms are needed to express ignoring or marginalizing over some outcomes of a causal process.
 \item The second class is the heart of the matter and consists of a \memph{causal mechanism} for each atomic variable $v$
      \[
      \pa(v) \to^{\kappa_v} v,
       \]
       where $\pa(v)$ is a chosen singular variable that contains all the parents of $v$,  and is more accurately denoted by $\pa_{G}(v)$ if necessary.  If $\pa(v) = \0$ then this is just $\0 \fun v$, which shall be called a \memph{exogenous} causal mechanism.
\end{itemize}

The causal theory $\cat{Cau}(G)$ is the \memph{free} Markov category generated from these two classes of morphisms (and the identity morphisms), by taking all compositions and products as depicted in (\ref{free:graph:def}) and (\ref{mon:free:graph:def}), subject only to the constraints in axioms (\ref{mon:law})-(\ref{discar:nat}).

For convenience, set $\kappa_\0 = 1_\0$ in $\cau(G)$.

A free category is a category generated from certain generators by well-defined operations in a ``no junk no noise'' manner: ``no junk'' in the sense that only those morphisms that can be so generated are in the category,  and ``no noise'' in the sense that no relations between morphisms hold unless they are required by the axioms.  For precise technical definitions and graphical constructions of free monoidal categories, see \cite{selinger2010survey}.  A graphical construction of free Markov categories takes a little more work,  which can be found in \cite{YinZhangModularity}.  For the present purposes,  we need not enter the rather technical details of the constructions, and we will simply use some lemmas from \cite{YinZhangModularity} in some of our proofs.

\begin{rem}\label{cau:mech:ind}
It may seem that the construction of $\cat{Cau}(G)$ depends on the choice of the singular variables $\pa(v)$ for the causal mechanisms $\kappa_v$. But this is not so: two distinct choices of $\pa(v)$ (and hence of $\kappa_v$) only differ by a permutation of atomic variables in $\pa(v)$ and the resulting  free  Markov categories are isomorphic.
\end{rem}

Following \cite{Fong:thesis} (see also \cite{jacobs2019causal}), we take $\cau(G)$ as an categorical embodiment of the syntax for causal reasoning with $G$. It can then be interpreted in any Markov category via strong Markov functors, yielding different kinds of causal models. For example, a CBN based on $G$ is a model of $\cat{Cau}(G)$ in the Markov category $\cat{FinStoch}$, the category containing stochastic matrices as morphisms \cite{ Fong:thesis, fritz2020synthetic, jacobs2019causal}, whereas a deterministic structural equation model based on $G$ is a model of $\cat{Cau}(G)$ in the Markov category $\cat{Set}$. We may also explore less studied causal models, such as possibilistic ones, which are models of $\cat{Cau}(G)$ in the Markov category $\cat{Rel}$, the morphisms in which are relations between sets \cite{Fong:thesis}.

\subsection{Causal effect morphisms}
\label{subsec:effect}

Recall the interventional probability distributions $P(\mathbf{V}|do(\mathbf{T}))$ in the context of CBN, which is often referred to as the causal effect of $\mathbf{T}$ on $\mathbf{V}$ \cite{tian2002general}. We now construct a class of morphisms in a causal theory that is a syntactic counterpart to such causal effects.

\begin{nota}\label{uni:pro}
In any strict Markov category $\cat M$ such as $\cat{Cau}(G)$, a morphism $A \fun B$ is called a \memph{multiplier} on  the monoidal product $A = \bigotimes_{i} A_i$ if it is generated from the duplicates, discards, symmetries, and identities on the factors $A_i$; so $B$ must be a monoidal product of copies of the factors $A_i$. Moreover, if $(\cat M_0, \otimes, \moi)$ is a free monoid and $A_1, \ldots, A_n$ are pairwise distinct objects in the alphabet  then  the multiplier is unique, which is denoted by $\iota_{A \rightarrow B}$. This is due to (\ref{mon:law}) or, more intuitively, coherence of  the graphical language for Markov categories (see \cite{YinZhangModularity}). For instance, if  $A = A_1 \otimes A_2 \otimes A_3$ and $B = A_1^2 \otimes A_2  \otimes A_1 \otimes A_2$ then $\iota_{A \rightarrow B}$ may be depicted as
$\begin{tikzpicture}[xscale = .3, yscale = .3, baseline={([yshift=8pt]current bounding box.south)}]
\begin{pgfonlayer}{nodelayer}
		\node [style=none] (108) at (4.275, 0.15) {};
		\node [style=none] (110) at (7.325, 0.15) {};
		\node [style=none] (111) at (5.275, -1.575) {};
		\node [style=none] (112) at (4.275, -0.1) {};
		\node [style=none] (113) at (6.275, -0.65) {};
		\node [style=none] (114) at (5.275, -0.95) {};
		\node [style=none] (117) at (6.175, 0.15) {};
		\node [style=none] (118) at (8.8, 0.15) {};
		\node [style=none] (119) at (7.975, -1.575) {};
		\node [style=none] (120) at (7.125, -0.6) {};
		\node [style=none] (121) at (8.775, -0.35) {};
		\node [style=none] (122) at (7.975, -1.075) {};
		\node [style=none] (124) at (10.2, -1.55) {};
		\node [style=none] (150) at (10.2, 0.025) {$\bullet$};
		\node [style=none] (151) at (5.275, 0.15) {};
		\node [style=none] (152) at (4.4, -1.5) {$A_1$};
		\node [style=none] (153) at (7, -1.55) {$A_2$};
		\node [style=none] (154) at (9.2, -1.5) {$A_3$};
	\end{pgfonlayer}
	\begin{pgfonlayer}{edgelayer}
		\draw [style=wire, in=-135, out=-90] (112.center) to (113.center);
		\draw [style=wire] (108.center) to (112.center);
		\draw [style=wire, in=45, out=-90, looseness=1.25] (110.center) to (113.center);
		\draw [style=wire, in=270, out=90] (111.center) to (114.center);
		\draw [style=wire, in=-90, out=-60, looseness=1.25] (120.center) to (121.center);
		\draw [style=wire, in=120, out=-90] (117.center) to (120.center);
		\draw [style=wire, in=90, out=-90, looseness=1.25] (118.center) to (121.center);
		\draw [style=wire] (119.center) to (122.center);
		\draw [style=wire] (124.center) to (150.center);
		\draw [style=wire] (114.center) to (151.center);
	\end{pgfonlayer}
\end{tikzpicture}$, where how the  duplicates in the trident are arranged, how the edges at the nodes are ordered, how the copies of the same object in the codomain are ordered, and so on, can all be left unspecified.
\end{nota}

Henceforth we work in $\cau(G)$.

\begin{ter}\label{what:dia}
By the construction in \cite{YinZhangModularity},  a morphism in  $\cau(G)$ is an equivalence class of string diagrams up to surgeries. Therefore, by a string diagram of a morphism, we mean any diagram in the equivalence class in question.
\end{ter}

\begin{nota}
Although, for our purpose, there is no need to distinguish between $wv$ and $vw$ in $\cat{Cau}(G)$, for a technical reason (symmetries in free Markov categories cannot be identities), we cannot work with the quotient of $W(V)$ with respect to the relation $wv=vw$ on words. This is also the reason why the maneuver in Remark~\ref{cau:mech:ind} is needed.

To remedy this, we first choose a total ordering on $V$ and denote the corresponding maximal singular variable by $\dot V = 1 \otimes \ldots \otimes n$. All singular variables we shall speak of are sub-variables of $\dot V$. For instance, if $v$, $w$ are  singular variables then  $v \cup w$  denotes the unique sub-variable of $\dot V$ that contains exactly the atomic variables in $v$, $w$.

The results below depend on the chosen ordering only because taking monoidal products of atomic variables does.
\end{nota}

\begin{defn}\label{gen:condi}
For singular variables $t$ and $v$, let $G_{t \rightarrow v}$ be the subgraph of $G$ that consists of all the vertices in $t \cup v$ and all the directed paths that end in $v$ but do not \memph{travel toward} $t$, that is, do not pass through or end in $t$ (starting in $t$ is allowed). Note that, for every  $i \in V(G_{t \rightarrow v})$, if $i \notin t$ then its parents in $G$ are all in $G_{t \rightarrow v}$ as well and if $i \in t$ then it has no parents in $G_{t \rightarrow v}$.

Construct a string diagram  as follows. For each $i \in V(G_{t \rightarrow v})$, let $\bar i$ be the monoidal product of as many copies of $i$ as the number of children of $i$ in $G_{t \rightarrow v}$. Let $\stri_i$ be a  string diagram of
\[
\begin{dcases*}
   \iota_{i \rightarrow \bar i \otimes i}                & if $i \in t \cap v$,\\
   \iota_{i \rightarrow \bar i}                            & if $i \in t \mi v$,\\
  \iota_{i \rightarrow \bar i \otimes i} \circ \kappa_i  & if $i \in v \mi t$,\\
  \iota_{i \rightarrow \bar i} \circ \kappa_i              & if $i \notin t \cup v$;
\end{dcases*}
\]
note  the extra copy of $i$ in the codomain of $\iota_{i \rightarrow \bar i \otimes i}$. According to Notation~\ref{uni:pro}, there is no need to choose orderings for the codomains of the multipliers employed here. For $j \in V(G_{t \rightarrow v})$, let $o_j$ be the number of output wires of $\stri_j$ and $p_j$ that of all the input wires labelled by $j$ of all the other  $\stri_i$. Observe that  $o_j = p_j + 1$ if $j \in v$ and  $o_j = p_j$ in all other cases. So we may connect the corresponding wires and fuse these \memph{components} $\stri_i$ into a single  string diagram, denoted by $\stri_{[v \| t]}$, whose input wires are labelled by $t$ and the output wires by $v$.  The string diagram thus obtained may not be unique up to isomorphisms, but the morphism it represents is, due to  coherence of  the graphical language for Markov categories. This morphism is referred to as the \memph{causal effect} of $t$ on $v$ and is denoted by $[v \| t] : t \fun v$, or simply $[v]$ when $t = \0$, which is also called the \memph{exogenous effect} on $v$.
\end{defn}

This class of morphisms was called ``causal conditionals'' in \cite[\S~4]{Fong:thesis}. We prefer to call them ``causal effects'' here because of their eponymous counterparts in probabilistic causal modeling mentioned earlier, but also because we will study a notion of a conditional in the next section, and $[v \| t]$ may not be a conditional in that sense.

\begin{exam}\label{exa:cau:con}
For any atomic variable $v$, if $\pa(v) = \0$ then $\kappa_v$ is simply depicted as $\begin{tikzpicture}[xscale = .3, yscale = .3, baseline={([yshift=6pt]current bounding box.south)}]
\begin{pgfonlayer}{nodelayer}
		\node [style=none] (124) at (9.825, -1.55) {$\bullet$};
		\node [style=none] (150) at (9.825, 0.025) {};
	\end{pgfonlayer}
	\begin{pgfonlayer}{edgelayer}
		\draw [style=wire] (124.center) to (150.center);
	\end{pgfonlayer}
\end{tikzpicture}$. The simplest causal effects are the ones of the form $[v \| \pa(v)]$, which is of course just the causal mechanism $\kappa_v$. Below are some simple examples of the causal effect $[z \| x]$ in $\cau(G)$ for four different graphs with three vertices:
\begin{center}
\begin{tabular}{ c|c|c|c|c }
The graph $G$                       & $\bfig \Atriangle(0,0)/<-`->`/<150, 200>[y`x`z; ``] \efig$ &
                                      $\bfig \Atriangle(0,0)/->`->`/<150, 200>[y`x`z; ``] \efig$ &
                                      $\bfig \Atriangle(0,0)/<-`<-`/<150, 200>[y`x`z; ``] \efig$ &
                                      $\bfig  \Atriangle(0,0)/->`->`->/<150, 200>[y`x`z; ``]  \efig$ \\
\hline
The subgraph $G_{x \rightarrow z}$  & $\bfig \Atriangle(0,0)/<-`->`/<150, 200>[y`x`z; ``] \efig$ &
                                      $\bfig \Atriangle(0,0)/`->`/<150, 200>[y`x`z; ``] \efig$ &
                                      $\bfig \Atriangle(0,0)/``/<150, 200>[`x`z; ``] \efig$ &
                                      $\bfig \Atriangle(0,0)/`->`->/<150, 200>[y`x`z; ``] \efig$ \\
\hline
The causal effect $[z \| x]$  & $\begin{tikzpicture}[xscale = .7, yscale = .4, baseline=(current  bounding  box.center)]
\begin{pgfonlayer}{nodelayer}
		\node [style=small box] (166) at (24.475, 1.7) {$\kappa_y$};
		\node [style=none] (167) at (24.475, 3.7) {};
		\node [style=none] (168) at (24.475, 2.125) {};
		\node [style=small box] (203) at (24.475, 4.05) {$\kappa_z$};
		\node [style=none] (204) at (24.475, 5.7) {};
		\node [style=none] (205) at (24.475, 4.475) {};
		\node [style=none] (206) at (24.475, 1.35) {};
		\node [style=none] (207) at (24.475, 0.125) {};
		\node [style=none] (209) at (24.075, 0.45) {$x$};
		\node [style=none] (210) at (24.075, 2.9) {$y$};
		\node [style=none] (211) at (24.075, 5.1) {$z$};
	\end{pgfonlayer}
	\begin{pgfonlayer}{edgelayer}
		\draw [style=wire] (167.center) to (168.center);
		\draw [style=wire] (204.center) to (205.center);
		\draw [style=wire] (206.center) to (207.center);
	\end{pgfonlayer}
\end{tikzpicture}$ & $\begin{tikzpicture}[xscale = .7, yscale = .5, baseline=(current  bounding  box.center)]
\begin{pgfonlayer}{nodelayer}
		\node [style=none] (167) at (24.475, 3.7) {};
		\node [style=none] (168) at (24.475, 2.45) {$\bullet$};
		\node [style=small box] (203) at (24.475, 4.05) {$\kappa_z$};
		\node [style=none] (204) at (24.475, 5.625) {};
		\node [style=none] (205) at (24.475, 4.475) {};
		\node [style=none] (206) at (23.225, 4.6) {$\bullet$};
		\node [style=none] (207) at (23.225, 3.375) {};
		\node [style=none] (209) at (22.825, 3.7) {$x$};
		\node [style=none] (210) at (24.075, 2.95) {$y$};
		\node [style=none] (211) at (24.075, 5.1) {$z$};
	\end{pgfonlayer}
	\begin{pgfonlayer}{edgelayer}
		\draw [style=wire] (167.center) to (168.center);
		\draw [style=wire] (204.center) to (205.center);
		\draw [style=wire] (206.center) to (207.center);
	\end{pgfonlayer}
\end{tikzpicture}$ & $\begin{tikzpicture}[xscale = .7, yscale = .5, baseline=(current  bounding  box.center)]
\begin{pgfonlayer}{nodelayer}
		\node [style=none] (204) at (24.525, 4.65) {};
		\node [style=none] (205) at (24.525, 3.4) {$\bullet$};
		\node [style=none] (206) at (23.225, 4.6) {$\bullet$};
		\node [style=none] (207) at (23.225, 3.375) {};
		\node [style=none] (209) at (22.825, 3.7) {$x$};
		\node [style=none] (211) at (24.2, 4.175) {$z$};
	\end{pgfonlayer}
	\begin{pgfonlayer}{edgelayer}
		\draw [style=wire] (204.center) to (205.center);
		\draw [style=wire] (206.center) to (207.center);
	\end{pgfonlayer}
\end{tikzpicture}$ & $\begin{tikzpicture}[xscale = .7, yscale = .5, baseline=(current  bounding  box.center)]
\begin{pgfonlayer}{nodelayer}
		\node [style=none] (17) at (5.7, 0.35) {};
		\node [style=none] (18) at (6.3, 0.35) {$\bullet$};
		\node [style=none] (19) at (6.3, 1.575) {};
		\node [style=none] (20) at (5.7, 1.575) {};
		\node [style=small box] (22) at (6, 2.075) {$\kappa_z$};
		\node [style=none] (24) at (6, 2.575) {};
		\node [style=none] (25) at (6, 3.775) {};
		\node [style=none] (37) at (5.65, 3.425) {$z$};
		\node [style=none] (38) at (5.3, 0.425) {$x$};
		\node [style=none] (39) at (6.7, 0.825) {$y$};
	\end{pgfonlayer}
	\begin{pgfonlayer}{edgelayer}
		\draw [style=wire] (17.center) to (20.center);
		\draw [style=wire] (18.center) to (19.center);
		\draw [style=wire] (24.center) to (25.center);
	\end{pgfonlayer}
\end{tikzpicture}$
\end{tabular}
\end{center}
Note that in the second and third examples, $x$ is not an ancestor of $z$ in the causal graph $G$, and the causal effect $[z\|x]$ is accordingly ``disconnected'': the morphism factors through the monoidal unit $\0$, which marks the lack of influence of $x$ on $z$.
\end{exam}

\section{The existence of conditionals and a generic do-calculus}
\label{sec:effect}

With string diagrams, we can use topological notions to aid reasoning. Here are some notions that will be used in some of the arguments below:
\begin{defn}\label{defn:inferen}
Let $\Gamma$ be a string diagram in a symmetric monoidal category $\cat C$. Denote the source of an edge $e$ in $\Gamma$ by $e(0)$ and the target by $e(1)$. A \memph{directed path} $a \rightsquigarrow b$ in $\Gamma$ is a sequence of edges in $\Gamma$
\[
p = (a = e_0,  e_1, \ldots, e_k,  b = e_{k+1})
\]
such that, for each $i$, $e_i(1) = e_{i+1}(0)$ and $e_i$ is not labeled by $\moi$; we also regard the source and target of each $e_i$ as belonging to the directed path, and write $p(0) = a(0)$ and $p(1) = b(1)$. A \memph{path} is just a concatenation of finitely many directed paths. In particular, a \memph{splitter path} (respectively, a \memph{collider path})  is a  concatenation of two directed paths joined at the starting nodes (respectively, at the ending nodes).

Two (not necessarily distinct) edges are \memph{connected} in $\Gamma$ if there is a path between them. More generally, two sets $A$, $B$ of edges are connected in $\Gamma$ if some $a \in A$ is connected with some $b \in B$, of particular interest is the case $A = \dom (\Gamma)$ and $B = \cod (\Gamma)$.
\end{defn}

In some proofs below, we shall need the theory on surgeries on string diagrams developed in \cite{YinZhangModularity}. For Markov categories, there are four types of surgeries, corresponding to the four diagrams in (\ref{mon:law}) and (\ref{discar:nat}), which shall be accordingly referred to as coassociativity surgery,  counitality surgery,  cocommutativity surgery,  and  discard surgery,  respectively.   Only the following portion from that theory is needed here.

Recall Remark~\ref{intro:string} and Terminology~\ref{what:dia}.  Let $\Gamma$ be a string diagram of a morphism in $\cau(G)$.  A node $x$ of  $\Gamma$ is \memph{quasi-terminal} if either it is a dead-end (discard) or every maximal directed path $p$ in $\Gamma$ with $p(0) = x$  runs into a dead-end or, in case that $x$ is a duplicate, this is so for all such paths through one of the prongs.  Denote the set of quasi-terminal nodes of $\Gamma$ by $\Delta_\Gamma$ and its complement by  $\tilde \Delta_\Gamma$.  Denote by $P_\Gamma$ the set of the directed paths that end in  $\cod(\Gamma)$ and by $S_\Gamma$ the set of splitter paths between edges in $\cod(\Gamma)$.

\begin{lem}\label{decor:mat}
Suppose that  $\Gamma$, $\Upsilon$ are string diagrams of the same morphism in $\cau(G)$. Then
\begin{itemize}
  \item There is a bijection $\pi : \tilde \Delta_\Gamma \fun \tilde \Delta_{\Upsilon}$ compatible with the labels in $\Gamma$, $\Upsilon$.
  \item There is a bijection $\dot \pi :  P_\Gamma \fun P_{\Upsilon}$  compatible with $\pi$, that is, $\pi$ restricts to a bijection between the nodes in $\tilde \Delta_\Gamma$ belonging to $p \in P_\Gamma$ and those in $\tilde \Delta_\Upsilon$ belonging to $\dot \pi(p)$.
  \item There is a bijection $\ddot \pi : S_\Gamma \fun S_{\Upsilon}$ compatible with $\pi$.
\end{itemize}
\end{lem}

We continue to work in $\cau(G)$. Here is a useful fact from  \cite{YinZhangModularity} that will be  needed in the subsequent arguments:

\begin{lem}\label{infe:ances}
Let $f : v \fun w$ be a morphism between singular variables. Then there is a morphism $g : v / v' \fun w$ for some sub-variable $v' \sub v$ such that $f = g \otimes \epsilon_{v'}$ and, in all string diagrams of $g$, every atomic variable in $v / v'$ is connected with an atomic variable in $w$ via a directed path.
\end{lem}

We now proceed to establish some results about causal effect morphisms in a causal theory. A central result has to do with the existence of conditionals in a Markov category, as is defined in \cite{fritz2020synthetic}.

\begin{defn}\label{defn:gen:cond}
Let $f : Z \fun X \otimes Y$ be a morphism in a Markov category $\cat M$.
\begin{itemize}
  \item The \memph{marginal} $f_{X|Z}$ of $f$ over $Y$ is the morphism $\begin{tikzpicture}[xscale = .5, yscale = .5, baseline={([yshift=10pt]current bounding box.south)}]
		\begin{pgfonlayer}{nodelayer}
		\node [style=none] (210) at (3.65, 2.55) {};
		\node [style=none] (211) at (3.65, 2.05) {};
		\node [style=small box] (225) at (3.65, 3.025) {$f$};
		\node [style=none] (226) at (3.425, 4.075) {};
		\node [style=none] (227) at (3.425, 3.525) {};
		\node [style=none] (228) at (3.875, 4.075) {$\bullet$};
		\node [style=none] (229) at (3.875, 3.525) {};
	\end{pgfonlayer}
	\begin{pgfonlayer}{edgelayer}
		\draw [style=wire] (210.center) to (211.center);
		\draw [style=wire] (226.center) to (227.center);
		\draw [style=wire] (228.center) to (229.center);
	\end{pgfonlayer}
\end{tikzpicture} = (1_X \otimes \epsilon_Y) \circ f: Z \fun X$.

  \item $f$ admits a \memph{conditional  over $X$} if there is a morphism $f_{Y|XZ} : X \otimes Z \fun Y$ such that
      \begin{equation}\label{cond:gen:mark}
        \begin{tikzpicture}[xscale = .6, yscale = .45, baseline=(current  bounding  box.center)]
         \begin{pgfonlayer}{nodelayer}
		\node [style=none] (110) at (7.875, 4.4) {};
		\node [style=none] (157) at (7.125, 0.625) {};
		\node [style=none] (158) at (7.125, 0.05) {};
		\node [style=none] (210) at (1.9, 2.55) {};
		\node [style=none] (211) at (1.9, 1.8) {};
		\node [style=wide small box] (225) at (1.9, 3.025) {$f$};
		\node [style=none] (226) at (1.425, 4.325) {};
		\node [style=none] (227) at (1.425, 3.525) {};
		\node [style=none] (228) at (2.375, 4.325) {};
		\node [style=none] (229) at (2.375, 3.525) {};
		\node [style=none] (273) at (6.35, 1.45) {};
		\node [style=wide small box] (274) at (7.35, 4.725) {$f_{Y|XZ}$};
		\node [style=none] (275) at (7.35, 6.025) {};
		\node [style=none] (276) at (7.35, 5.15) {};
		\node [style=wide small box] (281) at (6.35, 1.95) {$f$};
		\node [style=none] (283) at (5.875, 2.45) {};
		\node [style=none] (284) at (6.925, 3.15) {$\bullet$};
		\node [style=none] (285) at (6.825, 2.45) {};
		\node [style=none] (286) at (4.125, 2.925) {$=$};
		\node [style=none] (287) at (6.35, 1.175) {};
		\node [style=none] (288) at (7.85, 1.2) {};
		\node [style=none] (289) at (6.825, 4.2) {};
		\node [style=none] (290) at (6.075, 3.025) {};
		\node [style=none] (292) at (5.55, 3.375) {};
		\node [style=none] (293) at (6.55, 3.4) {};
		\node [style=none] (294) at (5.25, 5.5) {};
	\end{pgfonlayer}
	\begin{pgfonlayer}{edgelayer}
		\draw [style=wire] (157.center) to (158.center);
		\draw [style=wire] (210.center) to (211.center);
		\draw [style=wire] (226.center) to (227.center);
		\draw [style=wire] (228.center) to (229.center);
		\draw [style=wire] (275.center) to (276.center);
		\draw [style=wire, in=90, out=-90] (284.center) to (285.center);
		\draw [style=wire] (157.center) to (158.center);
		\draw [style=wire] (157.center) to (158.center);
		\draw [style=wire] (210.center) to (211.center);
		\draw [style=wire] (226.center) to (227.center);
		\draw [style=wire] (228.center) to (229.center);
		\draw [style=wire] (275.center) to (276.center);
		\draw [style=wire, in=90, out=-90] (284.center) to (285.center);
		\draw [style=wire, bend right=90, looseness=1.25] (287.center) to (288.center);
		\draw [style=wire] (287.center) to (273.center);
		\draw [style=wire, bend right=90, looseness=1.25] (292.center) to (293.center);
		\draw [style=wire, in=90, out=-90, looseness=1.25] (290.center) to (283.center);
		\draw [style=wire, in=-120, out=90] (293.center) to (289.center);
		\draw [style=wire, in=-90, out=90] (292.center) to (294.center);
		\draw [style=wire] (110.center) to (288.center);
		\draw [style=wire] (110.center) to (288.center);
	\end{pgfonlayer}
    \end{tikzpicture}
    \end{equation}
\end{itemize}

\end{defn}

The marginals of a causal effect morphism behave as expected.

\begin{lem}\label{three:cau:mar}
Let  $u$, $v$, and $w$ be singular variables with $v \cap w = \0$. Then the marginal of the causal effect $[vw \| u]$ over $v$ is the causal effect $[w \| u]$ (and that over $w$ is $[v \| u]$).
\end{lem}
\begin{proof}
By induction on the cardinality of $v$, this is immediately reduced to the case where $v$ is an atomic variable. We examine how composing with $\epsilon_{v}$ changes the component $\stri_v$ and other subsequently impacted components $\stri_i$  without changing the morphism represented (recall Definition \ref{gen:condi}). If $\stri_v$ has more than one output wire then, by  counitality surgery,  $\iota_{v \rightarrow \bar v \otimes v} $ is changed to $\iota_{v \rightarrow \bar v} $ and $\stri_v$ is thus changed without impacting any other $\stri_i$. If $\stri_v$ has only one output wire then, by discard surgery, it is reduced to $\bigotimes_{j \in \pa(v)} \epsilon_{j}$, where $\pa(v)$ is computed in $G_{u \rightarrow vw}$. In that case, for any $j \in \pa(v)$, we ask the same question that whether $\stri_j$ has more than one output wire or not, and proceed accordingly as before. Observe that, when there are no more surgeries to be performed, the remaining components $\stri_j$, including the modified ones, are exactly those required to construct $\stri_{[w \| u]}$. The lemma follows.
\end{proof}

Note that this lemma relies on the axiom (\ref{discar:nat}), or discard surgery, which is not imposed in \cite{Fong:thesis} as we do here. To see it, consider again the third graph in  Example~\ref{exa:cau:con} and the marginal $(1_x \otimes \epsilon_y) \circ [xy \| z]$ of the causal effect $[xy\|z]$ over $y$, then we have:
\begin{equation}\label{fong:counter}
\begin{tikzpicture}[xscale = .7, yscale = .5, baseline=(current  bounding  box.center)]
\begin{pgfonlayer}{nodelayer}
		\node [style=none] (17) at (5.75, 1.1) {};
		\node [style=none] (18) at (6.25, 0.45) {};
		\node [style=none] (19) at (6.25, 1.5) {};
		\node [style=none] (20) at (5.75, 1.5) {};
		\node [style=small box] (22) at (6, 2) {$\kappa_y$};
		\node [style=none] (23) at (4.425, 1.1) {};
		\node [style=none] (24) at (6, 2.5) {};
		\node [style=none] (25) at (6, 3.425) {$\bullet$};
		\node [style=none] (26) at (4.425, 3.025) {};
		\node [style=none] (27) at (5.1, 0.2) {$\bullet$};
		\node [style=none] (28) at (5.1, 0.6) {};
		\node [style=none] (29) at (3.925, 3.025) {$x$};
		\node [style=none] (30) at (6.65, 0.525) {$z$};
		\node [style=none] (37) at (6.525, 2.925) {$y$};
		\node [style=none] (40) at (8.225, 1.675) {$=$};
		\node [style=none] (212) at (13.075, 1.675) {$=$};
		\node [style=none] (213) at (14.725, 2.4) {};
		\node [style=none] (214) at (14.725, 1.15) {$\bullet$};
		\node [style=none] (215) at (15.925, 2.35) {$\bullet$};
		\node [style=none] (216) at (15.925, 1.125) {};
		\node [style=none] (217) at (16.275, 1.45) {$z$};
		\node [style=none] (218) at (14.4, 1.925) {$x$};
		\node [style=none] (219) at (11, 2.35) {};
		\node [style=none] (220) at (11.5, 1.2) {};
		\node [style=none] (221) at (11.5, 2.25) {$\bullet$};
		\node [style=none] (222) at (11, 2.5) {$\bullet$};
		\node [style=none] (224) at (9.675, 2.35) {};
		\node [style=none] (227) at (9.675, 2.525) {};
		\node [style=none] (228) at (10.35, 0.95) {$\bullet$};
		\node [style=none] (229) at (10.35, 1.85) {};
		\node [style=none] (230) at (9.95, 1.475) {$x$};
		\node [style=none] (231) at (11.9, 1.275) {$z$};
	\end{pgfonlayer}
	\begin{pgfonlayer}{edgelayer}
		\draw [style=wire] (17.center) to (20.center);
		\draw [style=wire] (18.center) to (19.center);
		\draw [style=wire] (23.center) to (26.center);
		\draw [style=wire] (24.center) to (25.center);
		\draw [style=wire, bend right=90, looseness=1.25] (23.center) to (17.center);
		\draw [style=wire] (27.center) to (28.center);
		\draw [style=wire] (213.center) to (214.center);
		\draw [style=wire] (215.center) to (216.center);
		\draw [style=wire] (219.center) to (222.center);
		\draw [style=wire] (220.center) to (221.center);
		\draw [style=wire] (224.center) to (227.center);
		\draw [style=wire, bend right=90, looseness=1.25] (224.center) to (219.center);
		\draw [style=wire] (228.center) to (229.center);
	\end{pgfonlayer}
\end{tikzpicture}
\end{equation}
where the first equality is begotten by discard surgery  and the second one by counitality surgery. Without (\ref{discar:nat}), the first equality would fail.

Related to this observation is a claim in \cite[Proposition~4.2]{Fong:thesis} that if $v$, $w$ are atomic variables and $v$ is not an ancestor of $w$ in $G$, then there exists no morphism $f: v \fun w$ in $\cat{Cau}(G)$ such that $v$ and $w$ are connected. Again, this is not quite right without (\ref{discar:nat}), as shown by the example in (\ref{fong:counter}). Now that we have imposed (\ref{discar:nat}), this claim does hold, and is immediate from Lemma \ref{infe:ances}:

\begin{prop}[\cite{Fong:thesis}]
\label{connect:ancestral}
Let $v$, $w$ be atomic variables. If there exists a morphism $v \fun w$ in which $v$, $w$ are connected then $v$ is an ancestor of $w$. Conversely, if $v$ is an ancestor of $w$ then they are connected by a directed path in $\stri_{[w \| v]}$.
\end{prop}

This fact signals that $\cau(G)$ is ``purely causal'', in that all connected morphisms in the category go from causal ancestors to descendants. As a result, merely ``associational'' or ``evidential'' relations are not expressed by any morphism in the category. (Recall the ``no junk'' property of a free category.)

In some Markov categories such as $\cat{FinStoch}$ mentioned earlier, every morphism of the form $f : Z \fun X \otimes Y$ admits conditionals over both objects in the codomain \cite{fritz2020synthetic}, but this is not the case in $\cat{Cau}(G)$. Take, for instance, the simple graph $x \rightarrow y$ and consider the exogenous effect $[xy] : \0 \fun xy$. If a conditional $[xy]_{x|y} : y \fun x$ over $y$ existed then we would have
\begin{equation}
\begin{tikzpicture}[xscale = .7, yscale = .5, baseline=(current  bounding  box.center)]
\begin{pgfonlayer}{nodelayer}
		\node [style=none] (17) at (5.75, 1.1) {};
		\node [style=none] (20) at (5.75, 1.5) {};
		\node [style=small box] (22) at (5.75, 2) {$\kappa_y$};
		\node [style=none] (23) at (4.425, 1.1) {};
		\node [style=none] (24) at (5.75, 2.5) {};
		\node [style=none] (25) at (5.75, 3.025) {};
		\node [style=none] (26) at (4.425, 3.025) {};
		\node [style=none] (27) at (5.1, 0.025) {$\bullet$};
		\node [style=none] (28) at (5.1, 0.6) {};
		\node [style=none] (29) at (3.9, 3.025) {$x$};
		\node [style=none] (37) at (6.25, 3) {$y$};
		\node [style=none] (48) at (7.825, 1.5) {};
		\node [style=none] (49) at (7.825, 1.5) {$=$};
		\node [style=none] (50) at (11.475, -0.65) {$\bullet$};
		\node [style=none] (53) at (11.475, -0.175) {};
		\node [style=small box] (54) at (11.475, 0.325) {$\kappa_y$};
		\node [style=none] (56) at (11.475, 0.825) {};
		\node [style=none] (57) at (11.475, 1.275) {};
		\node [style=none] (62) at (13.275, 3.375) {$y$};
		\node [style=none] (68) at (10.325, 2.15) {};
		\node [style=none] (69) at (12.65, 2.15) {};
		\node [style=none] (70) at (11.475, 1.275) {};
		\node [style=none] (72) at (10.325, 3.95) {};
		\node [style=none] (73) at (10.325, 3.35) {};
		\node [style=small box] (75) at (10.325, 2.85) {$[xy]_{x|y}$};
		\node [style=none] (76) at (10.325, 2.35) {};
		\node [style=none] (77) at (12.75, 3.375) {};
		\node [style=none] (78) at (10.825, 3.975) {$x$};
	\end{pgfonlayer}
	\begin{pgfonlayer}{edgelayer}
		\draw [style=wire] (17.center) to (20.center);
		\draw [style=wire] (23.center) to (26.center);
		\draw [style=wire] (24.center) to (25.center);
		\draw [style=wire, bend right=90, looseness=1.25] (23.center) to (17.center);
		\draw [style=wire] (27.center) to (28.center);
		\draw [style=wire] (50.center) to (53.center);
		\draw [style=wire] (56.center) to (57.center);
		\draw [style=wire, bend left=90, looseness=1.25] (69.center) to (68.center);
		\draw [style=wire] (72.center) to (73.center);
		\draw [style=wire, in=-90, out=90] (69.center) to (77.center);
		\draw [style=wire] (68.center) to (76.center);
	\end{pgfonlayer}
\end{tikzpicture}
\end{equation}
Since the duplicate $\delta_y$ does not occur on the left-hand side, by the first claim of Lemma~\ref{decor:mat}, its displayed occurrence on the right-hand side must be quasi-terminal,  but then $x$, $y$ cannot  be connected by a splitter path, violating the third claim of Lemma~\ref{decor:mat}. So the equality is not possible. On the other hand, $[xy]$ obviously admits a conditional over $x$, which is just $[y\| x] = \kappa_y$.

This simple example actually illustrates a general fact: for pairwise disjoint singular variables in $\cau(G)$, $u$, $v$,  $w$, if $[vw \| u]$ admits a conditional over $v$,  the conditional must be $[w \| uv]$.  We will leave the proof of this fact to another occasion, since it is a little involved and not directly relevant to the intended contributions of this paper.  For the present purpose, the directly relevant question is when $[w \| uv]$ is a conditional of $[w \| uv]$, or in other words, when the following \memph{decomposition} or \memph{disintegration} of a causal effect holds:

\begin{equation}\label{effect:decom}
\begin{tikzpicture}[xscale = .6, yscale = .45, baseline=(current  bounding  box.center)]
\begin{pgfonlayer}{nodelayer}
		\node [style=none] (110) at (7.875, 4.4) {};
		\node [style=none] (157) at (7.125, 0.625) {};
		\node [style=none] (158) at (7.125, 0.05) {};
		\node [style=none] (210) at (1.9, 2.55) {};
		\node [style=none] (211) at (1.9, 1.8) {};
		\node [style=wide small box] (225) at (1.9, 3.025) {$[vw \| u]$};
		\node [style=none] (226) at (1.425, 4.325) {};
		\node [style=none] (227) at (1.425, 3.525) {};
		\node [style=none] (228) at (2.375, 4.325) {};
		\node [style=none] (229) at (2.375, 3.525) {};
		\node [style=none] (273) at (6.35, 1.45) {};
		\node [style=wide small box] (274) at (7.35, 4.725) {$ [w\|vu] $};
		\node [style=none] (275) at (7.35, 6.025) {};
		\node [style=none] (276) at (7.35, 5.15) {};
		\node [style=wide small box] (281) at (6.35, 1.95) {$[v \| u]$};
		\node [style=none] (283) at (6.325, 2.45) {};
		\node [style=none] (286) at (4.125, 2.925) {$=$};
		\node [style=none] (287) at (6.35, 1.175) {};
		\node [style=none] (288) at (7.85, 1.2) {};
		\node [style=none] (289) at (6.825, 4.2) {};
		\node [style=none] (290) at (6.325, 3.025) {};
		\node [style=none] (292) at (5.8, 3.375) {};
		\node [style=none] (293) at (6.8, 3.4) {};
		\node [style=none] (294) at (5.25, 5.5) {};
	\end{pgfonlayer}
	\begin{pgfonlayer}{edgelayer}
		\draw [style=wire] (157.center) to (158.center);
		\draw [style=wire] (210.center) to (211.center);
		\draw [style=wire] (226.center) to (227.center);
		\draw [style=wire] (228.center) to (229.center);
		\draw [style=wire] (275.center) to (276.center);
		\draw [style=wire] (157.center) to (158.center);
		\draw [style=wire] (157.center) to (158.center);
		\draw [style=wire] (210.center) to (211.center);
		\draw [style=wire] (226.center) to (227.center);
		\draw [style=wire] (228.center) to (229.center);
		\draw [style=wire] (275.center) to (276.center);
		\draw [style=wire, bend right=90, looseness=1.25] (287.center) to (288.center);
		\draw [style=wire] (287.center) to (273.center);
		\draw [style=wire, bend right=90, looseness=1.25] (292.center) to (293.center);
		\draw [style=wire, in=90, out=-90, looseness=1.25] (290.center) to (283.center);
		\draw [style=wire] (293.center) to (289.center);
		\draw [style=wire, in=-90, out=90] (292.center) to (294.center);
		\draw [style=wire] (110.center) to (288.center);
		\draw [style=wire] (110.center) to (288.center);
	\end{pgfonlayer}
\end{tikzpicture}
\end{equation}
Call the property expressed by (\ref{effect:decom}) the \memph{decomposability} of $[vw\|u]$ over $v$. We now introduce some graphical conditions relevant to characterizing  decomposability and other significant concepts to be introduced presently:

\begin{defn}\label{def:separation}
Let  $i$, $j$ be two distinct vertices in $G$. A \memph{forward trek} from $i$ to $j$ in $G$ is a directed path from $i$ to $j$. A \memph{backward trek} from $i$ to $j$ is a directed path from $j$ to $i$, or a disjoint union of two directed paths joined at a distinct starting vertex $k$ (i.e., $i\leftarrow \dots \leftarrow k \rightarrow \dots \rightarrow j$).

Given $X, Y \sub V(G)$, a \memph{proper} forward (respectively, backward) trek from $X$ to $Y$ is a forward (respectively, backward) trek from some $i\in X$ to some $j\in Y$ that does not contain any other vertex in $X$ or in $Y$.

We say that
\begin{itemize}
  \item $X$ is \memph{forward-$t$-separated} from $Y$ by $Z$ if every proper forward trek in $G$ from $X$ to $Y$ contains some $k \in Z$;
  \item $X$ is \memph{backward-$t$-separated} from $Y$ by $Z$ if every proper backward trek in $G$ from $X$ to $Y$ contains some $k \in Z$;
  \item $X$ and $Y$ are \memph{$t$-separated} by $Z$ if $X$ is both forward-$t$-separated and backward-$t$-separated from $Y$ by $Z$.
\end{itemize}
\end{defn}

Observe that $t$-separation is a symmetric relation, but forward-$t$-separation and backward-$t$-separation are not.
Also note that $t$-separation is a simpler condition than the well-known $d$-separation \cite{Pearl1988}; the former is concerned only with blocking treks, whereas the latter also has an explicit requirement for paths that contain colliders, where two arrows point at the same vertex (i.e., $i\rightarrow k \leftarrow j$).

For the rest of this section, let $u$, $v$, and $w$ be pairwise disjoint singular variables in $\cau(G)$.

\begin{thm}\label{thm: main1}
$[vw \| u]$ is decomposable over $v$ if and only if $v$ is backward-$t$-separated from $w$ by $u$.
\end{thm}

\begin{proof}
For the ``if'' direction, suppose $v$ is backward-$t$-separated from $w$ by $u$, and we show that the equality (\ref{effect:decom}) holds. We examine the components $\stri_i$, $\stri_j$ for $i \in V(G_{u \rightarrow v})$, $j \in V(G_{uv \rightarrow w})$ and show that, together with $\delta_{v}$ and $\delta_{u}$ on the righthand side of (\ref{effect:decom}), they are exactly those needed to construct the string diagram $\stri_{[vw \| u]}$. There are the following cases.
\begin{itemize}
  \item $i \notin uv$. Then there is a directed path from $i$ to some $i' \in v$ in $G$ that does not pass through $u$. If $i$ also occurs in $G_{uv \rightarrow w}$ then either there is a directed path from it to some $j' \in w$ in $G$ that does not pass through $uv$, or it is contained in $w$, both of which are prohibited by the backward-$t$-separation of $v$ from $w$ by $u$. So  $\stri_i$ is the same in $\stri_{[vw \| u]}$ and $\stri_{[v \| u]}$.
  \item $i \in uv$. Let $i'$, $j'$ be any children of $i$ in $G_{u \rightarrow v}$, $G_{uv \rightarrow w}$, respectively. So $i' \notin u$ and $j' \notin uv$. If $i' \notin v$ then, by the case just considered, $i'$ does not occur in $G_{uv \rightarrow w}$ at all; for the same reason, $j'$ does not occur in $G_{u \rightarrow v}$. On the other hand, if $i' \in v$ then it cannot be a child  of $i$ in $G_{uv \rightarrow w}$. So $\stri_i$ in $\stri_{[vw \| u]}$ is the juxtaposition of  the two $\stri_i$ in $\stri_{[v \| u]}$ and $\stri_{[w \| uv]}$ joined by $\delta_{u}$.
  \item $j \notin uv$. Then $j$ is an ancestor of some $j' \in w$ in $G_{(uv \rightarrow w}$, and hence cannot occur in $G_{u \rightarrow v}$, again due to the backward-$t$-separation of $v$ from $w$ by $u$. So $\stri_j$ is the same in $\stri_{[vw \| u]}$ and $\stri_{[w \| uv]}$.
\end{itemize}
This establishes (\ref{effect:decom}).

For the ``only if'' direction, suppose that (\ref{effect:decom}) holds. Let $\pi$ be a proper backward trek from $a \in v$ to $b \in w$ in $G$. Suppose for contradiction that $\pi$ does not contain any vertex in $u$. By the first claim of Lemma~\ref{decor:mat},  $\kappa_b$ cannot occur  in $\stri_{[v \| u]}$. Thus, by the other two claims of Lemma~\ref{decor:mat}, $\pi$  would translate into  a directed path $b \rightsquigarrow a$ in $\stri_{[w \| v u]}$, which is not possible, or a splitter path between  $a$ and $b$ in $\stri_{[vw \| u]}$ that does not pass  through $v$ in the direction of $b$ and hence must pass through $u$, contradiction again.
\end{proof}

Readers familiar with Pearl's do-calculus \cite{Pearl1995} may have noticed the close affinity between backward-$t$-separation and the condition for Rule 2 of the do-calculus. Before we elaborate on the connection, let us introduce two more notions to fully match the do-calculus. One of them (``conditional independence'') is introduced in \cite{fritz2020synthetic} for all Markov categories.
\begin{defn} \label{def:irre:ind}
Let $\cat M$ be a Markov category.
\begin{itemize}
\item Let $f : X \otimes Z \fun Y$ be a morphism in $\cat M$. We say that $X$ is \memph{conditionally irrelevant to $Y$ given $Z$ over $f$} if there is a morphism $f_{Y|Z} : Z \fun Y$ such that
\begin{equation}\label{cau:ind:z}
     \begin{tikzpicture}[xscale = .6, yscale = .45, baseline=(current  bounding  box.center)]
    \begin{pgfonlayer}{nodelayer}
		\node [style=none] (242) at (0.825, 2.075) {};
		\node [style=wide small box] (254) at (0.825, 2.575) {$f_{Y|Z}$};
		\node [style=none] (255) at (0.8, 3.85) {};
		\node [style=none] (256) at (0.8, 3.075) {};
		\node [style=none] (260) at (-1.15, 3.425) {$\bullet$};
		\node [style=none] (261) at (-1.15, 1.525) {};
		\node [style=none] (271) at (0.825, 1.275) {};
		\node [style=none] (210) at (-4.3, 3.075) {};
		\node [style=none] (211) at (-4.3, 3.825) {};
		\node [style=wide small box] (225) at (-4.3, 2.6) {$f$};
		\node [style=none] (226) at (-3.825, 1.3) {};
		\node [style=none] (227) at (-3.825, 2.1) {};
		\node [style=none] (228) at (-4.775, 1.3) {};
		\node [style=none] (229) at (-4.775, 2.1) {};
		\node [style=none] (314) at (-2.3, 2.425) {$=$};
	\end{pgfonlayer}
	\begin{pgfonlayer}{edgelayer}
		\draw [style=wire] (210.center) to (211.center);
		\draw [style=wire] (226.center) to (227.center);
		\draw [style=wire] (228.center) to (229.center);
		\draw [style=wire] (255.center) to (256.center);
		\draw [style=wire] (260.center) to (261.center);
		\draw [style=wire] (210.center) to (211.center);
		\draw [style=wire] (226.center) to (227.center);
		\draw [style=wire] (228.center) to (229.center);
		\draw [style=wire] (242.center) to (271.center);
	\end{pgfonlayer}
     \end{tikzpicture}
    \end{equation}

\item Let $f : Z \fun X \otimes Y$ be a morphism in $\cat M$. We say that $X$, $Y$ are \memph{conditionally independent given $Z$ over $f$} if
   \begin{equation}\label{cond:ind}
     \begin{tikzpicture}[xscale = .6, yscale = .45, baseline=(current  bounding  box.center)]
    \begin{pgfonlayer}{nodelayer}
		\node [style=none] (241) at (-0.3, 2.525) {};
		\node [style=none] (242) at (1.9, 2.525) {};
		\node [style=wide small box] (254) at (1.9, 3.025) {$f$};
		\node [style=none] (255) at (2.375, 4.3) {};
		\node [style=none] (256) at (2.375, 3.525) {};
		\node [style=none] (257) at (1.425, 4.175) {$\bullet$};
		\node [style=none] (258) at (1.425, 3.525) {};
		\node [style=wide small box] (259) at (-0.3, 3.025) {$f$};
		\node [style=none] (260) at (0.175, 4.175) {$\bullet$};
		\node [style=none] (261) at (0.175, 3.525) {};
		\node [style=none] (262) at (-0.775, 4.3) {};
		\node [style=none] (263) at (-0.775, 3.525) {};
		\node [style=none] (271) at (1.9, 2.225) {};
		\node [style=none] (272) at (-0.3, 2.25) {};
		\node [style=none] (282) at (0.775, 0.725) {};
		\node [style=none] (210) at (-4.3, 2.05) {};
		\node [style=none] (211) at (-4.3, 1.3) {};
		\node [style=wide small box] (225) at (-4.3, 2.525) {$f$};
		\node [style=none] (226) at (-3.825, 3.825) {};
		\node [style=none] (227) at (-3.825, 3.025) {};
		\node [style=none] (228) at (-4.775, 3.825) {};
		\node [style=none] (229) at (-4.775, 3.025) {};
		\node [style=none] (314) at (-2.3, 2.425) {$=$};
		\node [style=none] (324) at (0.775, 1.425) {};
	\end{pgfonlayer}
	\begin{pgfonlayer}{edgelayer}
		\draw [style=wire] (210.center) to (211.center);
		\draw [style=wire] (226.center) to (227.center);
		\draw [style=wire] (228.center) to (229.center);
		\draw [style=wire] (255.center) to (256.center);
		\draw [style=wire] (257.center) to (258.center);
		\draw [style=wire] (260.center) to (261.center);
		\draw [style=wire] (262.center) to (263.center);
		\draw [style=wire, bend left=90, looseness=1.25] (271.center) to (272.center);
		\draw [style=wire] (210.center) to (211.center);
		\draw [style=wire] (226.center) to (227.center);
		\draw [style=wire] (228.center) to (229.center);
		\draw [style=wire] (242.center) to (271.center);
		\draw [style=wire] (282.center) to (324.center);
		\draw [style=wire] (241.center) to (272.center);
	\end{pgfonlayer}
     \end{tikzpicture}
    \end{equation}
\end{itemize}
\end{defn}

We consider a specialized version of conditional irrelevance:  $v$ is \memph{causally screened-off} from $w$ by $u$ if
\begin{equation}\label{screenoff}
\begin{tikzpicture}[xscale = .7, yscale = .45, baseline=(current  bounding  box.center)]
\begin{pgfonlayer}{nodelayer}
		\node [style=none] (242) at (0.825, 2.075) {};
		\node [style=wide small box] (254) at (0.825, 2.575) {$[w\|u]$};
		\node [style=none] (255) at (0.8, 3.85) {};
		\node [style=none] (256) at (0.8, 3.075) {};
		\node [style=none] (260) at (-1.15, 3.425) {$\bullet$};
		\node [style=none] (261) at (-1.15, 1.525) {};
		\node [style=none] (271) at (0.825, 1.275) {};
		\node [style=none] (210) at (-4.3, 3.075) {};
		\node [style=none] (211) at (-4.3, 3.825) {};
		\node [style=wide small box] (225) at (-4.3, 2.6) {$[w \| vu]$};
		\node [style=none] (226) at (-3.825, 1.3) {};
		\node [style=none] (227) at (-3.825, 2.1) {};
		\node [style=none] (228) at (-4.775, 1.3) {};
		\node [style=none] (229) at (-4.775, 2.1) {};
		\node [style=none] (314) at (-2.2, 2.425) {$=$};
	\end{pgfonlayer}
	\begin{pgfonlayer}{edgelayer}
		\draw [style=wire] (210.center) to (211.center);
		\draw [style=wire] (226.center) to (227.center);
		\draw [style=wire] (228.center) to (229.center);
		\draw [style=wire] (255.center) to (256.center);
		\draw [style=wire] (260.center) to (261.center);
		\draw [style=wire] (210.center) to (211.center);
		\draw [style=wire] (226.center) to (227.center);
		\draw [style=wire] (228.center) to (229.center);
		\draw [style=wire] (242.center) to (271.center);
	\end{pgfonlayer}
\end{tikzpicture}
\end{equation}
It is easy to show that causal screening-off is captured precisely by forward-$t$-separation.

\begin{thm}
\label{thm:main2}
$[w\| vu] = \epsilon_v\otimes [w\|u]$ if and only if $v$ is forward-$t$-separated from $w$ by $u$ in $G$.
\end{thm}
\begin{proof}
For the ``if'' direction, since $v$ is forward-$t$-separated from $w$ by $u$, no $i \in v$ can have children in $G_{vu \rightarrow w}$ and hence $G_{vu \rightarrow w}$ is the union of $G_{u \rightarrow w}$ and the trivial graph with vertices in $v$. It then follows from the construction of causal effects in Definition~\ref{gen:condi}  that (\ref{screenoff}) holds.

For the ``only if'' direction, suppose that (\ref{screenoff}) holds. If there is a proper forward trek from  $v$ to $w$ in $G$ that  does not contain any vertex in $u$ then, by Lemma~\ref{decor:mat}, it would translate into a directed path   on the right-hand side of (\ref{screenoff}) connecting $v$ and $w$, which is not possible.
\end{proof}

Similarly, conditional independence over causal effects is captured precisely by $t$-separation.

\begin{thm}
\label{thm:main3}
We have that $v$, $w$ are conditionally independent given $u$ over $[vw \| u]$ if and only if they are $t$-separated by $u$ in $G$.
\end{thm}

\begin{proof}
For the ``if'' direction, note that $t$-separation between $v$ and $w$ by $u$ entails that, on the one hand, $v$ is backward-$t$-separated from $w$ by $u$ and hence, by Theorem \ref{thm: main1}, the equality (\ref{effect:decom}) holds, and on the other hand,  $v$ is forward-$t$-separated from $w$ by $u$ and hence, by Theorem \ref{thm:main2}, the equality (\ref{screenoff}) holds. It then follows that
\begin{equation}
\begin{tikzpicture}[xscale = .65, yscale = .45, baseline=(current  bounding  box.center)]
\begin{pgfonlayer}{nodelayer}
		\node [style=none] (110) at (7.875, 4.4) {};
		\node [style=none] (157) at (7.125, 0.625) {};
		\node [style=none] (158) at (7.125, 0.05) {};
		\node [style=none] (210) at (1.9, 2.55) {};
		\node [style=none] (211) at (1.9, 1.8) {};
		\node [style=wide small box] (225) at (1.9, 3.025) {$[vw \| u]$};
		\node [style=none] (226) at (1.425, 4.325) {};
		\node [style=none] (227) at (1.425, 3.525) {};
		\node [style=none] (228) at (2.375, 4.325) {};
		\node [style=none] (229) at (2.375, 3.525) {};
		\node [style=none] (273) at (6.35, 1.45) {};
		\node [style=wide small box] (274) at (7.35, 4.725) {$ [w\|vu] $};
		\node [style=none] (275) at (7.35, 6.025) {};
		\node [style=none] (276) at (7.35, 5.15) {};
		\node [style=wide small box] (281) at (6.35, 1.95) {$[v \| u]$};
		\node [style=none] (283) at (6.325, 2.45) {};
		\node [style=none] (286) at (4.125, 2.925) {$=$};
		\node [style=none] (287) at (6.35, 1.175) {};
		\node [style=none] (288) at (7.85, 1.2) {};
		\node [style=none] (289) at (6.825, 4.2) {};
		\node [style=none] (290) at (6.325, 3.025) {};
		\node [style=none] (292) at (5.8, 3.375) {};
		\node [style=none] (293) at (6.8, 3.4) {};
		\node [style=none] (294) at (5.25, 5.5) {};
		\node [style=none] (295) at (13.375, 4.4) {};
		\node [style=none] (296) at (12.625, 0.625) {};
		\node [style=none] (297) at (12.625, 0.05) {};
		\node [style=none] (298) at (11.85, 1.45) {};
		\node [style=small box] (299) at (13.35, 4.725) {$ [w\|u] $};
		\node [style=none] (300) at (13.35, 6.025) {};
		\node [style=none] (301) at (13.35, 5.15) {};
		\node [style=wide small box] (302) at (11.85, 1.95) {$[v \| u]$};
		\node [style=none] (303) at (11.825, 2.45) {};
		\node [style=none] (304) at (9.625, 2.925) {$=$};
		\node [style=none] (305) at (11.85, 1.175) {};
		\node [style=none] (306) at (13.35, 1.2) {};
		\node [style=none] (307) at (12.325, 3.7) {$\bullet$};
		\node [style=none] (308) at (11.825, 3.025) {};
		\node [style=none] (309) at (11.3, 3.375) {};
		\node [style=none] (310) at (12.3, 3.4) {};
		\node [style=none] (311) at (10.75, 5.25) {};
		\node [style=none] (312) at (19.875, 2.875) {};
		\node [style=none] (313) at (18.875, 1.675) {};
		\node [style=none] (314) at (18.875, 1.1) {};
		\node [style=none] (315) at (17.85, 2.7) {};
		\node [style=wide small box] (316) at (19.85, 3.2) {$ [w\|u] $};
		\node [style=none] (317) at (19.85, 4.5) {};
		\node [style=none] (318) at (19.85, 3.625) {};
		\node [style=wide small box] (319) at (17.85, 3.2) {$[v \| u]$};
		\node [style=none] (320) at (17.825, 3.7) {};
		\node [style=none] (321) at (15.625, 2.925) {$=$};
		\node [style=none] (322) at (17.85, 2.425) {};
		\node [style=none] (323) at (19.85, 2.45) {};
		\node [style=none] (325) at (17.825, 4.525) {};
	\end{pgfonlayer}
	\begin{pgfonlayer}{edgelayer}
		\draw [style=wire] (157.center) to (158.center);
		\draw [style=wire] (210.center) to (211.center);
		\draw [style=wire] (226.center) to (227.center);
		\draw [style=wire] (228.center) to (229.center);
		\draw [style=wire] (275.center) to (276.center);
		\draw [style=wire] (157.center) to (158.center);
		\draw [style=wire] (157.center) to (158.center);
		\draw [style=wire] (210.center) to (211.center);
		\draw [style=wire] (226.center) to (227.center);
		\draw [style=wire] (228.center) to (229.center);
		\draw [style=wire] (275.center) to (276.center);
		\draw [style=wire, bend right=90, looseness=1.25] (287.center) to (288.center);
		\draw [style=wire] (287.center) to (273.center);
		\draw [style=wire, bend right=90, looseness=1.25] (292.center) to (293.center);
		\draw [style=wire, in=90, out=-90, looseness=1.25] (290.center) to (283.center);
		\draw [style=wire] (293.center) to (289.center);
		\draw [style=wire, in=-90, out=90] (292.center) to (294.center);
		\draw [style=wire] (110.center) to (288.center);
		\draw [style=wire] (110.center) to (288.center);
		\draw [style=wire] (296.center) to (297.center);
		\draw [style=wire] (300.center) to (301.center);
		\draw [style=wire] (296.center) to (297.center);
		\draw [style=wire] (296.center) to (297.center);
		\draw [style=wire] (300.center) to (301.center);
		\draw [style=wire, bend right=90, looseness=1.25] (305.center) to (306.center);
		\draw [style=wire] (305.center) to (298.center);
		\draw [style=wire, bend right=90, looseness=1.25] (309.center) to (310.center);
		\draw [style=wire, in=90, out=-90, looseness=1.25] (308.center) to (303.center);
		\draw [style=wire] (310.center) to (307.center);
		\draw [style=wire, in=-90, out=90] (309.center) to (311.center);
		\draw [style=wire] (295.center) to (306.center);
		\draw [style=wire] (295.center) to (306.center);
		\draw [style=wire] (313.center) to (314.center);
		\draw [style=wire] (317.center) to (318.center);
		\draw [style=wire] (313.center) to (314.center);
		\draw [style=wire] (313.center) to (314.center);
		\draw [style=wire] (317.center) to (318.center);
		\draw [style=wire, bend right=90, looseness=1.25] (322.center) to (323.center);
		\draw [style=wire] (322.center) to (315.center);
		\draw [style=wire, in=90, out=-90, looseness=1.25] (325.center) to (320.center);
		\draw [style=wire] (312.center) to (323.center);
		\draw [style=wire] (312.center) to (323.center);
	\end{pgfonlayer}
\end{tikzpicture}
\end{equation}
So, by Lemma~\ref{three:cau:mar},  $v$, $w$ are conditionally independent given $u$ over $[vw \| u]$.

For the ``only if'' direction, by Lemma~\ref{three:cau:mar} again, we may assume
\begin{equation}\label{concau:indep:free}
\begin{tikzpicture}[xscale = .6, yscale = .45, baseline=(current  bounding  box.center)]
\begin{pgfonlayer}{nodelayer}
		\node [style=none] (241) at (4.15, 2.525) {};
		\node [style=none] (242) at (1.95, 2.525) {};
		\node [style=none] (244) at (4.15, 2.25) {};
		\node [style=wide small box] (254) at (1.95, 3.025) {$[v \| u]$};
		\node [style=none] (255) at (1.975, 4.3) {};
		\node [style=none] (256) at (1.975, 3.525) {};
		\node [style=wide small box] (259) at (4.15, 3.025) {$[w \| u]$};
		\node [style=none] (262) at (4.125, 4.3) {};
		\node [style=none] (263) at (4.125, 3.525) {};
		\node [style=none] (271) at (1.95, 2.225) {};
		\node [style=none] (272) at (4.15, 2.25) {};
		\node [style=none] (282) at (3.075, 0.725) {};
		\node [style=none] (210) at (-2.1, 2.05) {};
		\node [style=none] (211) at (-2.1, 1.3) {};
		\node [style=wide small box] (225) at (-2.1, 2.525) {$[vw \| u]$};
		\node [style=none] (226) at (-2.575, 3.825) {};
		\node [style=none] (227) at (-2.575, 3.025) {};
		\node [style=none] (228) at (-1.625, 3.825) {};
		\node [style=none] (229) at (-1.625, 3.025) {};
		\node [style=none] (314) at (0, 2.425) {$=$};
		\node [style=none] (324) at (3.075, 1.425) {};
	\end{pgfonlayer}
	\begin{pgfonlayer}{edgelayer}
		\draw [style=wire] (210.center) to (211.center);
		\draw [style=wire] (226.center) to (227.center);
		\draw [style=wire] (228.center) to (229.center);
		\draw [style=wire] (241.center) to (244.center);
		\draw [style=wire] (255.center) to (256.center);
		\draw [style=wire] (262.center) to (263.center);
		\draw [style=wire, bend right=90, looseness=1.25] (271.center) to (272.center);
		\draw [style=wire] (210.center) to (211.center);
		\draw [style=wire] (226.center) to (227.center);
		\draw [style=wire] (228.center) to (229.center);
		\draw [style=wire] (242.center) to (271.center);
		\draw [style=wire] (282.center) to (324.center);
	\end{pgfonlayer}
\end{tikzpicture}
\end{equation}
Let $\pi$ be a proper forward or backward trek from $a \in v$ to $ b \in w$. If $\pi$ does not contain  any vertex in $u$ then it would translate into a directed or splitter path $\gamma$ between  $a$ and $b$ in $\stri_{[vw \| u]}$. Since $a, b \notin u$, we see that $\kappa_a$, $\kappa_b$ occur exactly once in $\stri_{[vw \| u]}$ and hence, by (\ref{concau:indep:free}) and the first claim of   Lemma~\ref{decor:mat}, $\kappa_a$ and hence $a$ do not occur in $\stri_{[w \| u]}$, whereas $\kappa_b$ and hence $b$ do not occur in $\stri_{[v \| u]}$. It follows from the other two claims of   Lemma~\ref{decor:mat} that $\gamma$ has to pass through $u$, contradiction.
\end{proof}

A merit of such theorems about the syntax category is that the sufficiency claims in them are immediately transferred to all models.

\begin{cor} \label{cor:model:do}
Let  $M$ be a Markov category and $F: \cau(G) \fun \cat M$ a strong Markov functor. Then
\begin{enumerate}
  \item If $v$ and $w$ are $t$-separated by $u$ in $G$ then, in $\cat M$, $F(v)$ and $F(w)$ are conditionally independent given $F(u)$ over $F([vw\|u])$:
  \begin{equation}\label{do-cal1}
     \begin{tikzpicture}[xscale = 1, yscale = .45, baseline=(current  bounding  box.center)]
    \begin{pgfonlayer}{nodelayer}
		\node [style=none] (241) at (-0.3, 2.525) {};
		\node [style=none] (242) at (1.9, 2.525) {};
		\node [style=wide small box] (254) at (1.9, 3.025) {$F([vw\|u])$};
		\node [style=none] (255) at (2.375, 4.3) {};
		\node [style=none] (256) at (2.375, 3.525) {};
		\node [style=none] (257) at (1.425, 4.175) {$\bullet$};
		\node [style=none] (258) at (1.425, 3.525) {};
		\node [style=wide small box] (259) at (-0.3, 3.025) {$F([vw\|u])$};
		\node [style=none] (260) at (0.175, 4.175) {$\bullet$};
		\node [style=none] (261) at (0.175, 3.525) {};
		\node [style=none] (262) at (-0.775, 4.3) {};
		\node [style=none] (263) at (-0.775, 3.525) {};
		\node [style=none] (271) at (1.9, 2.225) {};
		\node [style=none] (272) at (-0.3, 2.25) {};
		\node [style=none] (282) at (0.775, 0.725) {};
		\node [style=none] (210) at (-4.3, 2.05) {};
		\node [style=none] (211) at (-4.3, 1.3) {};
		\node [style=wide small box] (225) at (-4.3, 2.525) {$F([vw\|u])$};
		\node [style=none] (226) at (-3.825, 3.825) {};
		\node [style=none] (227) at (-3.825, 3.025) {};
		\node [style=none] (228) at (-4.775, 3.825) {};
		\node [style=none] (229) at (-4.775, 3.025) {};
		\node [style=none] (314) at (-2.3, 2.425) {$=$};
		\node [style=none] (324) at (0.775, 1.425) {};
	\end{pgfonlayer}
	\begin{pgfonlayer}{edgelayer}
		\draw [style=wire] (210.center) to (211.center);
		\draw [style=wire] (226.center) to (227.center);
		\draw [style=wire] (228.center) to (229.center);
		\draw [style=wire] (255.center) to (256.center);
		\draw [style=wire] (257.center) to (258.center);
		\draw [style=wire] (260.center) to (261.center);
		\draw [style=wire] (262.center) to (263.center);
		\draw [style=wire, bend left=90, looseness=1.25] (271.center) to (272.center);
		\draw [style=wire] (210.center) to (211.center);
		\draw [style=wire] (226.center) to (227.center);
		\draw [style=wire] (228.center) to (229.center);
		\draw [style=wire] (242.center) to (271.center);
		\draw [style=wire] (282.center) to (324.center);
		\draw [style=wire] (241.center) to (272.center);
	\end{pgfonlayer}
     \end{tikzpicture}
    \end{equation}

  \item If $v$ is backward-$t$-separated from $w$ by $u$ in $G$ then, in $\cat M$, $F([w\|uv])$ is a conditional of $F([vw\|u])$:
  \begin{equation}\label{do-cal2}
\begin{tikzpicture}[xscale = .8, yscale = .45, baseline=(current  bounding  box.center)]
\begin{pgfonlayer}{nodelayer}
		\node [style=none] (110) at (7.875, 4.4) {};
		\node [style=none] (157) at (7.125, 0.625) {};
		\node [style=none] (158) at (7.125, 0.05) {};
		\node [style=none] (210) at (1.9, 2.55) {};
		\node [style=none] (211) at (1.9, 1.8) {};
		\node [style=wide small box] (225) at (1.9, 3.025) {$F([vw \| u])$};
		\node [style=none] (226) at (1.425, 4.325) {};
		\node [style=none] (227) at (1.425, 3.525) {};
		\node [style=none] (228) at (2.375, 4.325) {};
		\node [style=none] (229) at (2.375, 3.525) {};
		\node [style=none] (273) at (6.35, 1.45) {};
		\node [style=wide small box] (274) at (7.35, 4.725) {$ F([w\|uv]) $};
		\node [style=none] (275) at (7.35, 6.025) {};
		\node [style=none] (276) at (7.35, 5.15) {};
		\node [style=wide small box] (281) at (6.35, 1.95) {$F([v \| u])$};
		\node [style=none] (283) at (6.325, 2.45) {};
		\node [style=none] (286) at (4.275, 2.925) {$=$};
		\node [style=none] (287) at (6.35, 1.175) {};
		\node [style=none] (288) at (7.85, 1.2) {};
		\node [style=none] (289) at (6.825, 4.2) {};
		\node [style=none] (290) at (6.325, 3.025) {};
		\node [style=none] (292) at (5.8, 3.375) {};
		\node [style=none] (293) at (6.8, 3.4) {};
		\node [style=none] (294) at (5.25, 5.5) {};
	\end{pgfonlayer}
	\begin{pgfonlayer}{edgelayer}
		\draw [style=wire] (157.center) to (158.center);
		\draw [style=wire] (210.center) to (211.center);
		\draw [style=wire] (226.center) to (227.center);
		\draw [style=wire] (228.center) to (229.center);
		\draw [style=wire] (275.center) to (276.center);
		\draw [style=wire] (157.center) to (158.center);
		\draw [style=wire] (157.center) to (158.center);
		\draw [style=wire] (210.center) to (211.center);
		\draw [style=wire] (226.center) to (227.center);
		\draw [style=wire] (228.center) to (229.center);
		\draw [style=wire] (275.center) to (276.center);
		\draw [style=wire, bend right=90, looseness=1.25] (287.center) to (288.center);
		\draw [style=wire] (287.center) to (273.center);
		\draw [style=wire, bend right=90, looseness=1.25] (292.center) to (293.center);
		\draw [style=wire, in=90, out=-90, looseness=1.25] (290.center) to (283.center);
		\draw [style=wire] (293.center) to (289.center);
		\draw [style=wire, in=-90, out=90] (292.center) to (294.center);
		\draw [style=wire] (110.center) to (288.center);
		\draw [style=wire] (110.center) to (288.center);
	\end{pgfonlayer}
\end{tikzpicture}
\end{equation}

  \item If $v$ is forward-$t$-separated from $w$ by $u$ in $G$ then, in $\cat M$, $F(v)$ is conditionally irrelevant to $F(w)$ given $F(u)$ over $F([w\|uv])$:
  \begin{equation}\label{do-cal3}
     \begin{tikzpicture}[xscale = .8, yscale = .45, baseline=(current  bounding  box.center)]
    \begin{pgfonlayer}{nodelayer}
		\node [style=none] (242) at (0.825, 2.075) {};
		\node [style=wide small box] (254) at (0.825, 2.575) {$F([w\|u])$};
		\node [style=none] (255) at (0.8, 3.85) {};
		\node [style=none] (256) at (0.8, 3.075) {};
		\node [style=none] (260) at (-1.15, 3.425) {$\bullet$};
		\node [style=none] (261) at (-1.15, 1.525) {};
		\node [style=none] (271) at (0.825, 1.275) {};
		\node [style=none] (210) at (-4.3, 3.075) {};
		\node [style=none] (211) at (-4.3, 3.825) {};
		\node [style=wide small box] (225) at (-4.3, 2.6) {$F([w\|uv])$};
		\node [style=none] (226) at (-3.825, 1.3) {};
		\node [style=none] (227) at (-3.825, 2.1) {};
		\node [style=none] (228) at (-4.775, 1.3) {};
		\node [style=none] (229) at (-4.775, 2.1) {};
		\node [style=none] (314) at (-2.1, 2.425) {$=$};
	\end{pgfonlayer}
	\begin{pgfonlayer}{edgelayer}
		\draw [style=wire] (210.center) to (211.center);
		\draw [style=wire] (226.center) to (227.center);
		\draw [style=wire] (228.center) to (229.center);
		\draw [style=wire] (255.center) to (256.center);
		\draw [style=wire] (260.center) to (261.center);
		\draw [style=wire] (210.center) to (211.center);
		\draw [style=wire] (226.center) to (227.center);
		\draw [style=wire] (228.center) to (229.center);
		\draw [style=wire] (242.center) to (271.center);
	\end{pgfonlayer}
     \end{tikzpicture}
    \end{equation}

  \end{enumerate}
\end{cor}

On the other hand, the necessity claims do not hold in all models. For example, the decomposition property (\ref{do-cal2}) always holds in deterministic structural equation models (as functors from $\cau(G)$ to $\cat{Set}$), regardless of backward-$t$-separation. The necessity in question is necessity for validity (``true in all models''), rather than truth in particular models.

\section{The causal core of the do-calculus}
\label{sec:cal}
Corollary \ref{cor:model:do} is particularly interesting because its three clauses correspond to the three rules in Pearl's do-calculus, respectively. Suppose $\cat M = \cat{FinStoch}$, and $F$ sends each causal mechanism $\kappa_v$ to a positive stochastic matrix, so that it gives rise to a causal Bayesian network (CBN) model \cite{Fong:thesis}, in which the pre-intervention joint probability distribution is positive (which is assumed by Pearl's do-calculus). Let $\mathbf{X}$ denote the set of random variables represented by the object $u$, $\mathbf{Y}$ by $w$, and $\mathbf{Z}$ by $v$. Then equations (\ref{do-cal1})-(\ref{do-cal3}) can be reformulated as follows.\footnote{In $\cat{FinStoch}$, we can simply use nonzero natural numbers as the objects, and the morphisms are stochastic matrices: a morphism $n\fun m$ is a $m\times n$ stochastic matrix. A discard morphism $\epsilon_n: n\fun 1$ is the $1\times n$ stochastic matrix in which each entry is $1$, and a duplicate morphism $\delta_n : n \fun n^2$ is the $n^2\times n$ stochastic matrix in which $(\delta_n)^{(i - 1)n + i}_{i} = 1$ (and other entries are zero.) The composition of morphisms is given by matrix multiplication, and the monoidal product is given by the Kronecker product of matrices \cite{Fong:thesis,fritz2020synthetic}. As shown in \cite{Fong:thesis}, $F([v_G])$, the image of the exogenous effect on $v_G$ in $\cat{FinStoch}$, is a stochastic matrix (in fact, a column vector) encoding a joint probability distribution over the set of random variables $\mathbf(V)$ represented by $v_G$ that satisfies the factorization in (\ref{fac:pre}). His argument can be generalized to show that $F([v\|u])$ is a stochastic matrix encoding the distributions over the random variables represented by $v$ given that those represented by $u$ are intervened to take various values, according to the intervention principle (\ref{fac:post}). Note that equations (\ref{prob-docal-1})-(\ref{prob-docal-3}) are understood as holding for all values of $\mathbf{X}, \mathbf{Y}, \mathbf{Z}$, and so express equality between specific entries in the relevant matrices.}

Equation (\ref{do-cal1}) is rendered as:
\begin{equation}
\label{prob-docal-1}
P(\mathbf{Y}, \mathbf{Z} | do(\mathbf{X})) = P(\mathbf{Y} | do(\mathbf{X}))P(\mathbf{Z} | do(\mathbf{X})).
\end{equation}

Equation (\ref{do-cal2}) is rendered as:
\begin{equation}
\label{prob-docal-2}
P(\mathbf{Y}, \mathbf{Z} | do(\mathbf{X})) = P(\mathbf{Z}| do(\mathbf{X}))P(\mathbf{Y} | do(\mathbf{X}), do(\mathbf{Z})).
\end{equation}

Equation (\ref{do-cal3}) is rendered as:  \begin{equation}
\label{prob-docal-3}
P(\mathbf{Y}| do(\mathbf{X}), do(\mathbf{Z})) = P(\mathbf{Y}| do(\mathbf{X})).
\end{equation}

Note in addition that by the probability calculus and the assumed positivity, (\ref{prob-docal-1}) is equivalent to \[
P(\mathbf{Y} | do(\mathbf{X}), \mathbf{Z}) = P(\mathbf{Y} | do(\mathbf{X})),
\]
and (\ref{prob-docal-2}) is equivalent to \[
P(\mathbf{Y} | do(\mathbf{X}), do(\mathbf{Z})) = P(\mathbf{Y} | do(\mathbf{X}), \mathbf{Z}),
\] because by the chain rule of the probability calculus, $P(\mathbf{Y}, \mathbf{Z} | do(\mathbf{X})) = P(\mathbf{Z}| do(\mathbf{X}))P(\mathbf{Y} | do(\mathbf{X}), \mathbf{Z})$.

The upshot is that Corollary \ref{cor:model:do}, when applied to a CBN model based on $G$ (with a positive or regular pre-intervention distribution), entails the following rules.
\begin{itemize}
    \item[{\it Rule 1}] (Insertion/deletion of observations): if $\mathbf{Y}$ and $\mathbf{Z}$ are $t$-separated by $\mathbf{X}$ in $G$, then \[ P(\mathbf{Y} | do(\mathbf{X}), \mathbf{Z}) = P(\mathbf{Y} | do(\mathbf{X})). \]

    \item[{\it Rule 2}] (Action/observation exchange): if $\mathbf{Z}$ is backward-$t$-separated from $\mathbf{Y}$ by $\mathbf{X}$ in $G$, then \[ P(\mathbf{Y} | do(\mathbf{X}), do(\mathbf{Z})) = P(\mathbf{Y} | do(\mathbf{X}), \mathbf{Z}). \]

    \item[{\it Rule 3}] (Insertion/deletion of actions): if $\mathbf{Z}$ is forward-$t$-separated from $\mathbf{Y}$ by $\mathbf{X}$ in $G$, then \[ P(\mathbf{Y} | do(\mathbf{X}), do(\mathbf{Z})) = P(\mathbf{Y} | do(\mathbf{X})). \]

\end{itemize}

Pearl's do-calculus \cite{Pearl1995} consists of exactly three rules like these, but each rule therein is more general than the corresponding rule above and is formulated in terms of the more complicated $d$-separation criterion and various modifications of $G$. The extra generality in Pearl's version is that the consequent equation in each rule has an extra set of variables $\mathbf{W}$ to be conditioned upon on both sides of the equation. For example, the consequent equation in the first rule of Pearl's do-calculus is \[ P(\mathbf{Y} | do(\mathbf{X}), \mathbf{Z}, \mathbf{W}) = P(\mathbf{Y} | do(\mathbf{X}), \mathbf{W}), \]
and similarly for the other two rules. It is a simple exercise to check that each of the rules above is exactly equivalent to the corresponding rule in Pearl's calculus when $\mathbf{W}$ is taken to be empty.

So Corollary \ref{cor:model:do}, when applied to a CBN model, yields a specialized version of Pearl's do-calculus. However, although each rule in the specialized version is a special case of the corresponding rule in the full version, taken together they are actually as strong as the full version. To see this, it suffices to show that we can recover the intervention principle (\ref{fac:post}) from the specialized rules, or to be more exact, from Rule 2 and Rule 3 above, since Rule 1 is entailed by the conjunction of Rule 2 and Rule 3 (just as in the full version, see \cite{HV2006}.) Since the full version is entailed by the intervention principle (plus the probability calculus), it is also entailed by the specialized version (plus the probability calculus) if the intervention principle is entailed by the specialized version (plus the probability calculus).

We now sketch a fairly simple argument to that effect. It helps to first consider the pre-intervention case, where we need to show that Rule 2 and Rule 3 entail that the pre-intervention probability distribution factorizes as in (\ref{fac:pre}). This is equivalent to deriving the local Markov condition from Rule 2 and Rule 3, the condition that every variable is probabilistically independent of its non-descendants conditional on its parents \cite{SGS2000}. The derivation is straightforward. For any random variable $V\in \mathbf{V}$ in the CBN, by Rule 2, we have
\[
P(V | do(\pa(V))) = P(V | \pa(V)),
\]
and
\[
P(V | do(\nd(V))) = P(V | \nd(V)),
\]
where $\pa(V)$ and $\nd(V)$ denote the set of $V$'s parents in $G$ and the set of $V$'s non-descendants in $G$ (i.e., those variables of which $V$ is not an ancestor), respectively. This is so because $\pa(V)$ is trivially backward-$t$-separated from $V$ by the empty set (for there is no proper backward trek from $\pa(V) $ to $V$), and so is $\nd(V)$, which contains $\pa(V)$ as a subset. Then by Rule 3, we have
\[
P(V | do(\nd(V))) = P(V | do(\pa(V))),
\]
simply because every forward trek to $V$ contains a parent of $V$. It follows that for each $V$, \[
P(V | \nd(V)) = P(V | \pa(V)).
\]
So the factorization required by the intervention principle holds in the pre-intervention case. This argument easily generalizes to any post-intervention probability distribution $P(\mathbf{V}|do(\mathbf{T}))$; that is, we can derive in the same fashion from Rule 2 and Rule 3 that for every $V$,
\[
P(V | \nd^*(V), do(\mathbf{T})) = P(V | \pa^*(V), do(\mathbf{T})),
\]
where $\pa^*(V)$ and $\nd^*(V)$ denote the set of $V$'s parents and the set of $V$'s non-descendants, respectively, in the subgraph of $G$ in which all arrows into variables in $\mathbf{T}$ are deleted.
From this follows the factorization of $P(\mathbf{V}|do(\mathbf{T}))$ required by the intervention principle (\ref{fac:post}).

Therefore, the full do-calculus can in principle be derived from the specialized do-calculus together with the probability calculus. This fact reveals an equivalent formulation of the do-calculus for CBN models that is simpler than the standard formulation. More importantly, this simpler formulation reflects the ``causal core'' of the do-calculus, for it is an instance of the generic do-calculus given in Corollary \ref{cor:model:do}, and the generic do-calculus is derived from results in a syntax category that is, so to speak, purely causal (because there is no morphism in that category to match non-causal or evidential relations between variables.) More precisely, the causal core of the do-calculus is given by the rule for causal decomposition ((\ref{do-cal2}), rendered as (\ref{prob-docal-2}) in a CBN model) and the rule for causal screening-off ((\ref{do-cal3}), rendered as (\ref{prob-docal-3}) in a CBN model). These are derived without any consideration of the non-causal features of the models. The standard do-calculus for the CBN models can be seen as derived from a conjunction of these two rules on the one hand, which are purely causal, and the probability calculus on the other hand, which is non-causal.

\section{Conclusion}
\label{sec:con}
Following the pioneering work of \cite{Fong:thesis}, we studied the causal effect morphisms in a causal DAG-induced free Markov category in some depth, and established sufficient and necessary graphical conditions for some conceptually important properties of such morphisms, including especially decomposition and screening-off. Our results yield a generic do-calculus that is more general and abstract than the standard do-calculus in the context of causal Bayesian networks. Not only is the generic do-calculus more widely applicable, it is also conceptually illuminating in that it reveals the purely causal component of the do-calculus. When applied to causal Bayesian networks, it also results in a simpler but equivalent formulation of the probabilistic do-calculus.

Since the simpler do-calculus uses trek-separation rather than the more convoluted d-separation, it is probably easier to explain and understand.  Moreover, the simpler do-calculus may also be readily extendable to other causal graphical models derived from DAG models. For example,  in \cite{Zhang2008}, Pearl's do-calculus is extended to the so-called partial ancestral graphs (PAGs), which are used to represent Markov equivalence classes of DAG models.  The extension is intended to capture the applicability of a do-calculus rule in all DAGs in the equivalence class represented by a PAG, but due to the complex graphical conditions in Pearl's do-calculus,  it only accommodates some but not all such cases of unanimous applicability.  We suspect that an extension of the simpler do-calculus highlighted in this paper would be more straightforward and complete.

\section*{Acknowledgements}
This research was supported by the Research Grants Council of Hong Kong
under the General Research Fund 13602818.

\providecommand{\bysame}{\leavevmode\hbox to3em{\hrulefill}\thinspace}
\providecommand{\MR}{\relax\ifhmode\unskip\space\fi MR }
\providecommand{\MRhref}[2]{%
  \href{http://www.ams.org/mathscinet-getitem?mr=#1}{#2}
}
\providecommand{\href}[2]{#2}

\end{document}